\documentclass{article}

\usepackage{microtype}
\usepackage{graphicx}
\usepackage{subcaption}
\usepackage{booktabs} 

\usepackage{hyperref}

\usepackage[preprint]{icml2026}

\usepackage{amsmath}
\usepackage{amssymb}
\usepackage{mathtools}
\usepackage{amsthm}

\usepackage{thm-restate}

\usepackage[capitalize,noabbrev]{cleveref}

\theoremstyle{plain}
\newtheorem{theorem}{Theorem}[section]

\theoremstyle{definition}

\newtheorem{assumption}[theorem]{Assumption}
\theoremstyle{remark}
\newtheorem{remark}[theorem]{Remark}

\usepackage[textsize=tiny]{todonotes}

\newcommand{\figwidthtwo}{0.48\textwidth}
\newcommand{\figwidththree}{0.30\textwidth}

\newcommand{\figwidthfour}{0.225\textwidth}

\newcommand{\figwidthfive}{0.18\textwidth}

\input{definitions}

\icmltitlerunning{Gradient Residual Connections}

\begin{document}

\twocolumn[
  \icmltitle{Gradient Residual Connections}



  \icmlsetsymbol{equal}{*}

  \begin{icmlauthorlist}
    \icmlauthor{Yangchen Pan}{yyy}
    \icmlauthor{Qizhen Ying}{yyy}
    \icmlauthor{Philip Torr}{yyy}
    \icmlauthor{Bo Liu}{comp}
  \end{icmlauthorlist}

  \icmlaffiliation{yyy}{Department of Engineering Science, University of Oxford}
  \icmlaffiliation{comp}{Electrical and Computer Engineering Department, University of Arizona}

  \icmlcorrespondingauthor{Yangchen Pan}{yangchen.pan@eng.ox.ac.uk}

  \icmlkeywords{Machine Learning, ICML}

  \vskip 0.3in
]



\printAffiliationsAndNotice{}  

\begin{abstract}
Existing work has linked properties of a function’s gradient to the difficulty of function approximation. Motivated by these insights, we study how gradient information can be leveraged to improve neural networks' ability to approximate high-frequency functions, and we propose a gradient-based residual connection as a complement to the standard identity skip connection used in residual networks. 
We provide simple theoretical intuition for why gradient information can help distinguish inputs and improve the approximation of functions with rapidly varying behavior. On a synthetic regression task with a high-frequency sinusoidal ground truth, we show that conventional residual connections struggle to capture high-frequency patterns. In contrast, our gradient residual substantially improves approximation quality. We then introduce a convex combination of the standard and gradient residuals, allowing the network to flexibly control how strongly it relies on gradient information. After validating the design choices of our proposed method through an ablation study, we further validate our approach's utility on the single-image super-resolution task, where the underlying function may be high-frequency. Finally, on standard tasks such as image classification and segmentation, our method achieves performance comparable to standard residual networks, suggesting its broad utility. 
\end{abstract}

\section{Introduction}\label{sec:intro}
There has been sustained interest in using a frequency-based perspective to study the generalization and learning behavior of machine learning algorithms. A dataset (or the target function it induces) can be viewed as a signal composed of components at different frequencies. From this viewpoint, high-frequency components are often harder to learn or reconstruct, since they typically require more samples to approximate accurately—for example, a rapidly oscillating sinusoidal function compared to a low-frequency one \cite{smale2004shannon,smale2005shannon2}.

A close line of work studies the learning dynamics of neural networks through a frequency lens. A common observation is that neural networks trained with gradient-based methods tend to fit low-frequency components earlier (or faster) than high-frequency components, and several works provide empirical and theoretical interpretations of this phenomenon \cite{rahaman2019spectralbias,xu2018fprinciplegeneralloss,xu2020fprinciple,xu2021deepfreq,basri2019convergenceratefreq}. Relatedly, \citet{cao2021spectralbias} provides more specific analysis about certain spectral components can be learned more efficiently than others. Motivated by the difficulty of learning high-frequency functions, some architectures explicitly incorporate sinusoidal structure; for instance, \citet{sitzmann2020siren} use periodic activation functions to better represent high-frequency signals.


Among those work studying learning behavior from a frequency angle, the most closely related is \citet{pan2020dynafreq}, which connects learning difficulty and frequency by proposing a computationally convenient proxy for local frequency of a function around a point (e.g., via gradient or Hessian magnitude). This makes frequency not only an analysis tool, but also an implementable principle for data acquisition and prioritization. Such viewpoint is also loosely connected to classical uncertainty-type intuitions in harmonic analysis \cite{stein2003fourier}: functions that are highly localized in the input domain typically require a broader spread of Fourier components for accurate representation, which could imply involving higher-frequency content.

In practice, learning a high-frequency function from real-world datasets often requires capturing sophisticated input–output relationships, which heavily rely on the use of a powerful function approximator such as a deep neural network. However, training a deep neural network can be nontrivial and challenging. Among the most popular techniques are skip connections, particularly residual networks, which are widely adopted due to their simplicity and effectiveness. Residual connections have been used across a broad range of applications, from basic image classification \cite{he2016deepresidual} to solving reinforcement learning tasks such as AlphaGo \cite{silver2016mastering}.

We would like to emphasize that in our work, \emph{high-frequency information} refers to learning or approximating a high-frequency function: for example, $\sin(8\pi)$ varies much more rapidly than $\sin(2\pi)$. This is different from some works in vision \cite{ma2023understanding,si2024freeu,zhang2024residualharm} that discuss high-frequency image details, where high frequency refers to the input signal itself (e.g., edges/textures/high frequency components by fourier transformation of an image), rather than the underlying function being approximated. Nevertheless, it is reasonable that some tasks may require modeling both aspects.

Our key intuition is that a function’s gradients reflect its local frequency and can be highly informative for distinguishing points in high-frequency regions. This motivates our central research question: \emph{how can we explicitly incorporate gradient information into a neural network to better approximate high frequency functions?} We address this by introducing \emph{gradient residual connections}, which enable the network to approximate high-frequency functions more easily, while remaining competitive on standard benchmarks.

We begin with reviewing background on the relationship between gradient information and learning difficulty, and discuss how such information can be exploited in deep learning. We then introduce our main approach, termed gradient residual connections, in which residual paths incorporate gradients of intermediate representations with respect to the earlier layers, with the goal of enhancing sensitivity to high-frequency structure. We further provide theoretical insights into why gradient vectors can help distinguish data points in high-frequency regions, extending existing theoretical connections between gradient magnitude and learning difficulty. Finally, we conduct extensive experiments on both synthetic and real-world datasets to validate our design choices and the effectiveness of our method. In particular, we show clear benefits on super-resolution tasks, while also achieving performance comparable to conventional residual networks on standard image classification and segmentation benchmarks. We conclude by summarizing our contributions and discussing limitations and directions for future work.

\section{Background}\label{sec:bg}
This section introduces background on conventional residual connections in deep neural networks, as well as existing work that leverages frequency-based perspectives to understand function approximation.

\subsection{Residual Connections}

Optimizing deep neural networks requires nontrivial effort, as it introduces additional challenges compared with training shallow models or linear functions. These challenges range from architecture-dependent initialization, architecture design, and optimization algorithm design (e.g., gradient-based, adaptive-gradient, or non-gradient methods), to issues such as vanishing gradients or gradient explosion (the latter typically arising in RNNs). As a result, a variety of techniques have been developed to tackle these challenges.

Among these, one of the most popular techniques is \emph{residual networks}~\cite{srivastava2015highway,srivastava2015training,he2016deepresidual,he2016identity}, which aim to improve the optimization of very deep networks by enhancing gradient flow. The core idea is the use of skip (or shortcut) connections. Let \( \Fvec(\cdot): \mathbb{R}^d \rightarrow \mathbb{R}^d \) be a mapping between two \(d\)-dimensional representation spaces (downsampling or upsampling may be adopted to ensure matching dimensionality). Residual networks are characterized by the operation
\[
\mathbf{h}(\mathbf{x}) \defeq \Fvec(\mathbf{x}) + \mathbf{x}, \quad \mathbf{x} \in \mathbb{R}^d,
\]
where \(\mathbf{x}\) may be the raw input or an intermediate representation produced by earlier layers. The key point is that this composition \(\mathbf{h}\) is used as the input for subsequent layers, instead of \(\Fvec(\mathbf{x})\) alone. This design effectively reduces, or allows the network to flexibly control, the amount of information that must be captured by \(\Fvec\), which is particularly important when many layers of computation are involved and gradient flow would otherwise become very small or vanish.


\subsection{Frequency and Function Approximation}

A frequency-based perspective is a useful way to interpret learning difficulty. In particular, ~\citet{pan2020dynafreq} define, for any function $f:\mathbb{R}^n\to\mathbb{R}$, a \emph{local} function on the unit ball $B(x,1)=\{y:\|y-x\|<1\}$ and its \emph{local Fourier transform}
$
\hat f(k)=\int_{B(x,1)} f(y)\,e^{-2\pi i\, y^\top k}\,dy.
$
Assuming finite local energy $\int_{B(x,1)} f(y)^2\,dy<\infty$, they introduce the normalized local frequency distribution
$
\pi_f(k)=\frac{\lvert \hat f(k)\rvert^2}{\int_{\mathbb{R}^n}\lvert \hat f(\kappa)\rvert^2\,d\kappa}.
$
The theorem then links local smoothness to local spectral content: the squared gradient (respectively, Hessian) norm integrated over $B(x,1)$ equals the local function energy multiplied by the \emph{second} (respectively, \emph{fourth}) moment of $\pi_f$ (up to constants), i.e., the average $\|k\|^2$ (respectively, $\|k\|^4$) under $\pi_f$. This connection is consistent with the uncertainty principle \cite{stein2003fourier} intuition: stronger spatial localization of $f$ within a neighborhood $B(x,1) $(i.e. the function energy is more concentrated around the centre $x$) necessitates greater high-frequency content, and thus larger local derivative norms. 

Related ideas of utilizing the connection between gradient information and learning difficulty also appear in the area of active learning. For instance, Variance of Gradients~\cite{agarwal2022vog} characterizes example difficulty using statistics of per-sample gradients over training, while BADGE~\cite{ash2020badge} selects informative and diverse samples via gradient-based embeddings. 

Inspired by this line of work, we introduce in Section~\ref{sec:mainapproach} our gradient residual connections and introduce intuitive baselines in Section~\ref{sec:sinexp} to encourage the network to capture higher-frequency information by using gradient information.

\section{Gradient-based Residual Connections}\label{sec:mainapproach}
This section introduces our main approach—gradient-based residuals—and provides theoretical insights into how such gradient vectors can be used to distinguish points in the high-frequency region. 

\subsection{Gradient-based Residual}\label{sec:gd-residual}

We know that the standard residual connection adds an identity mapping from an earlier representation to a later one after some transformations, i.e., \(\xvec \rightarrow \Fvec(\xvec) + \xvec\).
Building on existing results linking gradient information to the difficulty of learning high-frequency functions, we propose leveraging the gradient vector itself as a representation residual added to the earlier layer.
Given \(\xvec \in \RR^d\) and \(\Fvec: \RR^d \to \RR^d\), i.e., the mapping from an earlier layer to a later one, we formally propose the following gradient-based shortcut:
{\setlength{\abovedisplayskip}{4pt}
 \setlength{\belowdisplayskip}{4pt}
\begin{small}
\begin{equation}\label{eq-gdshortcut}
    \hvec_g(\xvec) \defeq \Fvec(\xvec) + \xvec + \sum_{i=1}^d \nabla \Fvec_i(\xvec),
\end{equation}
\end{small}
}
where we define \(\nabla \Fvec(\xvec)\)---the gradient w.r.t. $\xvec$---as the Jacobian matrix with dimension \(d\times d\), and \(\sum_{i=1}^d \nabla \Fvec_i(\xvec) \in \RR^d\) sums over rows of the Jacobian and has the same dimensionality as \(\xvec \in \RR^d\).
Intuitively, this design attempts to leverage a representation that indicates how sensitive the overall feature vector \(\Fvec(\cdot)\) is to each input dimension.

However, there are several sensible drawbacks. First, the gradient itself can be highly sensitive, potentially resulting in instability of the training process.
Second, if the underlying function we are trying to learn is not highly frequency-dependent, this design may be harmed by overfitting to irrelevant variation. Therefore, we propose a simple modification to both stabilize and broaden the utility of such shortcuts: a \emph{convex combination of skips}.
{\setlength{\abovedisplayskip}{4pt}
 \setlength{\belowdisplayskip}{4pt}
\begin{small}
\begin{equation}\label{eq-convexgdresidual}
    \hvec(\xvec) \defeq \Fvec(\xvec) + (1-\sigma(\alpha)) \xvec + \sigma(\alpha) \sum_{i=1}^d \nabla  \Fvec_i(\xvec),
\end{equation}
\end{small}
}
where $\alpha$ is a trainable scalar parameter and $\sigma(\cdot)$ is the sigmoid function so that $\sigma(\alpha)$ interpolates between the standard residual and the gradient-based shortcut. We emphasize that this is a natural and minimal design choice rather than a strict requirement. In practice, the formulation may depend on the application domain and the neural network architecture. For more sophisticated networks, both the standard residual connection and the gradient residual connection may be necessary; in such cases, one might introduce independent trainable parameters for each term instead of using a convex combination. We include more extensive studies about this design choice in Section \ref{sec:appendix-srexp}. 


\begin{remark}
\citet{pan2020dynafreq} takes the gradient norm of the measure of learning difficulty, instead of the gradient vector itself; in fact, our theoretical insight \ref{sec:gd-theory} shows that gradient vector could be a useful representation used to distinguish points, independent of the norm of the gradient. 
\end{remark}

\begin{remark}
One might consider an alternative design that aggregates gradients across input dimensions (of \(\xvec\)) rather than across feature dimensions of \(\Fvec\).
However, this type of computation is not amenable to automatic differentiation, as it typically requires explicit loops over feature dimensions and can be significantly more costly.
More importantly, our choice
$
\sum_{i=1}^d \nabla \Fvec_i(\xvec) \;=\; \nabla_\xvec \Big( \sum_{i=1}^d \Fvec_i(\xvec) \Big)
$
is the gradient of a well-defined scalar functional of the features.
This preserves a clear interpretation of the shortcut as a single ``sensitivity direction'' in input space induced by \(\Fvec\), whereas aggregating over input dimensions would in general no longer correspond to the gradient of any scalar potential and would blur this geometric meaning. 
\end{remark}

In practice, to further stabilize the gradient vector, we normalize it to have a controlled norm by using
$
\frac{\sum_{i=1}^d \nabla \Fvec_i(\xvec)}{\big\|\sum_{i=1}^d \nabla \Fvec_i(\xvec)\big\|_2}, 
$
where a zero denominator can be avoided by adding a small constant. 

\subsection{Theoretical Insight}\label{sec:gd-theory}

The goal of this section is to provide mathematical intuition for why the (normalized) gradient vector can be a useful representation for distinguishing points in the high-frequency part of a function. Informally, this is because two points can have nearly opposite gradient directions even though they are extremely close in Euclidean distance.
We begin with introducing some definitions and assumptions. Detailed proofs are in the Appendix \ref{sec:appendix-theory}. 


\textbf{Define concerned function} $f:\RR^d\to\RR$ be $C^1$ with Fourier transform $\fhat$ (under the convention
$\fhat(\xivec)=\int_{\RR^d} f(\xvec)e^{-i\xvec\cdot\xivec}\,d\xvec$ and
$f(\xvec)=\frac{1}{(2\pi)^d}\int_{\RR^d}\fhat(\xivec)e^{i\xvec\cdot\xivec}\,d\xivec$). One might think of this as the function we are trying to approximate in a machine learning task. 

\textbf{Define high frequency component of $f$.} Fix a unit vector $\uvec\in\RR^d$, a frequency $\omega>0$, and bandwidth $\delta>0$. Define the band
\begin{small}
\begin{align}
B(\omega\uvec,\delta) \defeq \{\xivec\in\RR^d:\ \|\xivec-\omega\uvec\|\le \delta\},
\end{align}
\end{small}
based on which we can use the band-pass mask: 
\begin{small}
\begin{align}
\fhat_h(\xivec) \defeq \fhat(\xivec)\,\mathbf{1}_{B(\omega\uvec,\delta)}(\xivec),
\qquad
f_h \defeq \mathcal{F}^{-1}(\fhat_h),
\end{align}
\end{small}
and set $f_l\defeq f-f_h$ (so $f=f_l+f_h$ always).

For conciseness, define the band mass and constant
\begin{small}
\begin{equation}\label{eq:def_M_Cd}
M \defeq \int_{B(\omega\uvec,\delta)} \|\xivec\|\,|\fhat(\xivec)|\,d\xivec,
\qquad
C_d \defeq \frac{\pi}{(2\pi)^d}.
\end{equation}
\end{small}

\textbf{Define two points}, which we will show that when they are arbitrarily close to each other, their gradient directions can be quite different. Let $\xvec_0\in\RR^d$ satisfy
\begin{equation}\label{eq:assump_highgrad}
\|\nabla f_h(\xvec_0)\|\ge L
\end{equation}
for some $L>0$, and define another point
$
\xvec_1 \defeq \xvec_0 + \frac{\pi}{\omega}\uvec.
$

\begin{assumption}[Finite mass]\label{assm:finitemass}
\begin{equation}\label{eq:assump_moment1}
\int_{\RR^d} \|\xivec\|\,|\fhat(\xivec)|\,d\xivec<\infty .
\end{equation}
\end{assumption}

\begin{assumption}[Bounded gradient norm]\label{assm:boundedgd} 
For some sufficiently small $\varepsilon>0$,
\begin{equation}\label{eq:assump_lowgrad}
\sup_{\xvec\in\RR^d}\|\nabla f_l(\xvec)\|\le \varepsilon L.
\end{equation}
\end{assumption}

\begin{remark}
Assumption \ref{assm:finitemass} is standard in learning theory when one seeks approximation guarantees for neural networks. \ref{assm:boundedgd} is not strong, as it only bounds the gradient of the low-frequency part of the function. 
\end{remark}

Based on these definitions and assumptions, we are ready to introduce two lemmas, which were further used to prove our main theorem \ref{thm:main-thm}.

\begin{restatable}{lemma}{highgdsumbound}[Upper bound of high frequency component gradient.]\label{lemma:eq:high_sum_bound}
The gradient of $f_h$ at $\xvec_0, \xvec_1$ satisfies: 
\begin{small}
\begin{equation}\label{eq:high_sum_bound}
\|\nabla f_h(\xvec_0)+\nabla f_h(\xvec_1)\|
\le C_d\frac{\delta}{\omega}M.
\end{equation}
\end{small}
\end{restatable}

\begin{restatable}{lemma}{gdboundtwopoints}[Lower bound of gradients.]\label{lemma:eq:den_lb_x1x0}
The gradient of $f$ at $\xvec_0, \xvec_1$ satisfies: 
\begin{small}
\begin{align}\label{eq:den_lb_x0}
\| \nabla f(\xvec_0) \| \ge (1-\varepsilon)L, 
\end{align}
\begin{equation}\label{eq:den_lb_x1}
     \|\nabla f(\xvec_1)\| \ge (1-\varepsilon)L - C_d\frac{\delta}{\omega}M.
\end{equation}
\end{small}
\end{restatable}

\begin{assumption}[Dominance of high frequency gradient]\label{assm:positive-norm}
\begin{small}
\begin{equation}
    (1-\varepsilon)L > C_d\frac{\delta}{\omega}M, \textit{or, } L - C_d\frac{\delta}{\omega}M > \varepsilon L.
\end{equation}
\end{small}
\end{assumption}

\begin{remark}
The immediate purpose of Assumption~\ref{assm:positive-norm} is to ensure that the bounds in  \eqref{eq:den_lb_x1} and \eqref{eq:main_bound} are positive and non-vacuous. Beyond this technical requirement, the condition admits a clear interpretation. 

Imagine that if $f_h$ was a pure plane wave (i.e., sin/cos/function with a single frequency) of frequency $\omega\uvec$, we would have $\nabla f_h(\xvec_1)=-\nabla f_h(\xvec_0)$. However, signals are never pure --- they contain low-frequency background and have finite bandwidth. 
As shown in \eqref{eq:high_sum_bound}:
$\|\nabla f_h(\xvec_0)+\nabla f_h(\xvec_1)\| \le C_d\frac{\delta}{\omega}M,$
the term $C_d\frac{\delta}{\omega}M$ quantifies the maximal deviation from exact sign reversal due to the finite bandwidth. 

Therefore, Assumption~\ref{assm:positive-norm} ensures that even after the shift from $\xvec_0$ to $\xvec_1$, the high‑frequency oscillation remains the dominant feature of the gradient---strong enough to overcome both the low‑frequency “noise” and the distortion effect of the band. That is, $\|\nabla f_h\|$ is still large (at least $L - C_d\frac{\delta}{\omega}M$) and dominates $\|\nabla f_l\|$. One might better understand this by going through the proof of \eqref{eq:den_lb_x1}. 
\end{remark}

\begin{restatable}{theorem}{maintheorem}[Gradient reversal.]\label{thm:main-thm}
Under assumptions \ref{assm:finitemass}, \ref{assm:boundedgd}, $
\gvec_0\defeq \frac{\nabla f(\xvec_0)}{\|\nabla f(\xvec_0)\|}, 
\gvec_1\defeq \frac{\nabla f(\xvec_1)}{\|\nabla f(\xvec_1)\|}
$, we have

1) \begin{small}\begin{equation}\label{eq:main_bound}
\|\gvec_0+\gvec_1\|
\ \le\
\frac{2\Big(C_d\frac{\delta}{\omega}M + 2\varepsilon L\Big)}
{(1-\varepsilon)L - C_d\frac{\delta}{\omega}M}.
\end{equation}
\end{small}
2) In particular, whenever $\varepsilon\to 0$ and $\delta/\omega\to 0$ (with $M/L$ controlled), the right-hand side of \eqref{eq:main_bound} is
$O\!\left(\frac{\delta}{\omega}\frac{M}{L}+\varepsilon\right)$ and hence $\|\gvec_0+\gvec_1\|=o(1)$.

3) Moreover, for any $\eta\in(0,\pi)$, if the right-hand side of \eqref{eq:main_bound} is $\le 2\sin(\eta/2)$, then the angles of the two gradient vectors satisfy: $\angle(\gvec_0,\gvec_1)\ge \pi-\eta.$
\end{restatable}
\emph{Proof sketch.} The left-hand side of \eqref{eq:main_bound} can be upper bounded using an elementary inequality $\left\|\frac{\avec}{\|\avec\|}+\frac{\bvec}{\|\bvec\|}\right\|
\le \frac{2\|\avec+\bvec\|}{\min(\|\avec\|,\|\bvec\|)}$, which reduces the problem to upper-bounding the numerator and lower-bounding the denominator. Upper bounding the numerator requires expressing the gradient of $f$ at a given point as a sum of gradients of $f_l$ and $f_h$, which relies on Lemma~\ref{lemma:eq:high_sum_bound}. 

\textbf{Interpretation of the theorem.}
The shift $\|\xvec_1-\xvec_0\|=\pi/\omega$ shows that in a high-frequency regime (large $\omega$), the two points can be close in Euclidean distance. 
At the same time, provided the low-frequency gradient contribution is small (small $\varepsilon$) and the band is narrow relative to its center frequency (small $\delta/\omega$, with $M/L$ controlled), the normalized gradients satisfy 
$\|\gvec_0+\gvec_1\| = O\!\left(\varepsilon + \frac{\delta}{\omega}\right)$,
and in particular $\|\gvec_0+\gvec_1\|=o(1)$ under these asymptotics.
Equivalently, the gradient directions at $\xvec_0$ and $\xvec_1$ become nearly opposite. 
This supports the use of the (normalized) gradient vector as a scale-invariant representation to distinguish points that lie in regions dominated by narrow-band high-frequency content.

\textbf{Intuitive \& simplified example.} Consider the function $\sin(\omega, \uvec^\top \xvec)$
where $\omega > 0$ denotes the frequency of the function. Consider two points: $\xvec_1 = \xvec$ and $\xvec_2 = \xvec + \frac{\pi}{\omega} \uvec$. The distance between these two points approaches zero as $\omega \to \infty$, since $\|\xvec_1 - \xvec_2\| = \frac{\pi}{\omega}.$
However, the gradients at these two points are opposite as long as $\cos(\omega, \uvec^\top \xvec) \neq 0$. 
Recognizing that functions encountered in practical machine learning settings are often far more complex than simple plane-wave functions such as sine or cosine, Theorem~\ref{thm:main-thm} provides a rigorous characterization of how close one gradient vector can be to the \emph{opposite} of another, as well as the conditions under which this occurs for general functions, rather than for simple sinusoidal cases.

\section{Synthetic Experiments}\label{sec:sinexp}
This section presents an empirical study on a one-dimensional, easily visualizable synthetic dataset. Despite its simplicity, the experimental design remains rigorous, enabling extensive hyperparameter sweeps and averaging over a sufficiently large number of random seeds. The setup is minimal yet nontrivial and captures key challenges encountered in realistic settings.

The goals are to: (i) verify the capability of gradient-based connections in approximating high-frequency functions; (ii) verify that the benefit comes from the gradient representation rather than from a trainable scaling factor; and (iii) ablate design choices such as using normalization or not and using second-order information or not. 

It should be noted that, this experiment does not challenge the benefits of standard residual connections in very deep networks. Rather, it contrasts their underlying motivations with those of gradient-based residual connections by comparing both in appropriately deep networks, where differences in high-frequency approximation become clearer.

\textbf{Setup.} We generate inputs $\xvec$ uniformly on $[-4\pi,4\pi]$ with $N$ samples. The ground-truth function is defined by
{\setlength{\abovedisplayskip}{4pt}
 \setlength{\belowdisplayskip}{4pt}
\begin{small}
\begin{equation}
y^\star(x)=
\begin{cases}
\sin(0.5x)+0.5\sin(2.5x), & x<0,\\
\sin(2x)+0.5\sin(7x), & x\ge 0.
\end{cases}
\end{equation}
\end{small}
}
We add Gaussian noise with standard deviation $0.1$ to obtain the training target $y=y^\star(x)+\varepsilon$, where $\varepsilon\sim\mathcal{N}(0,0.1^2)$. We generate $3000$ data points and train for $5000$ epochs. This experiment is intended to be minimal yet nontrivial. First, noise is added because practical data are typically noise-contaminated and the gradient vector in our approach may be sensitive to noise. Second, we train for sufficient number of epochs to avoid conclusions based solely on early training. We use a three-hidden-layer neural network, with the number of neurons per hidden layer $d\in\{16,32,64\}$. Training is performed using stochastic gradient descent with a mini-batch size of $512$. Due to the low dimensionality of this dataset, all gradient-based residual methods compute gradients with respect to the first hidden-layer representation rather than the raw input. Note that, unless otherwise specified, no gradient is further taken from $\sum_{i=1}^{d}\nabla \Fvec_i(\xvec)$ during training. 

\textbf{Algorithms and naming.} 
\textbf{Regular} refers to the baseline without any residual connection. 
\textbf{Grad residual} uses only the gradient-based connection
\(
\hvec_g(\xvec) \defeq \Fvec(\xvec) + \sum_{i=1}^{d}\nabla \Fvec_i(\xvec),
\)
without the conventional residual connection. 
\textbf{Standard residual} denotes the usual residual connection, including a variant with a \emph{trainable scalar} multiplying the identity branch, i.e., $\Fvec(\mathbf{x}) + \alpha \mathbf{x}$ where $\alpha$ is a trainable scalar. 
\textbf{Convex combined residual} is our main approach, defined in \eqref{eq-convexgdresidual}. 
\textbf{Addition residual} simply adds the standard residual and gradient residual, as in \eqref{eq-gdshortcut}. 
Finally, \textbf{Grad magnitude} computes the same gradient term \(\sum_{i=1}^{d}\nabla \Fvec_i(\xvec)\) but uses only its magnitude: the pre-output representation is formed by concatenating the final hidden-layer representation with the scalar
\(
\left\|\sum_{i=1}^{d}\nabla \Fvec_i(\xvec)\right\|_2,
\)
i.e., \(\mathrm{activation}\!\left[\hvec(\xvec),\ \left\|\sum_{i=1}^{d}\nabla \Fvec_i(\xvec)\right\|_2\right]\).
This variant is inspired by \citet{pan2020dynafreq}, as reviewed earlier.

Figure~\ref{fig:sinlc} shows how different algorithms perform as neural network capacity increases. There are several important observations and corresponding conclusions to draw.

\emph{The utility of gradient connection}. Under relatively limited neural network capacity ($d=16$), the variants with gradient residual (\textbf{Grad Residual} and \textbf{Convex-combined residual}) perform best, and are substantially better than the others. The variant that concatenates the gradient norm (\textbf{Grad Magnitude}) is slightly worse, but still better than the rest. However, as network capacity increases, the benefit of \textbf{Grad Magnitude} vanishes and may even become harmful, since its magnitude can vary significantly, whose variance may outweigh representational benefit. Second, the standard residual approach is significantly worse than using the gradient residual alone, even though use a trainable scalar to scale the residual connection. This clearly indicates that the gains come from leveraging additional gradient-based representation, rather than simply introducing a trainable scalar.

Third, standard residual connections can become harmful on this domain, as the shortcut may allow lower-level, less expressive representations to interfere with the final, more sophisticated representation. Indeed, when $d = 64$, we checked that the trainable scalar on the standard residual connection decreases from the order of $10^{-1}$ to $10^{-2}$ within the first $100$ epochs, indicating that the model tends to suppress this connection. This observation is also supported by the learning curve of \textbf{regular} in Figure \ref{fig:sinlc} (b)(c), which significantly outperform the standard residual connection without a trainable scalar in \textcolor{green}{\textbf{dashed}} green line. 

On the other hand, Figure \ref{fig:sinlc} (a-c) also suggests that as network capacity increases, the advantage of our method over the regular one gradually diminishes, which is expected since the model becomes less reliant on gradient-based information due to its own increasing expressive power. 

\emph{This observation suggests a potential tradeoff between a network’s ability to approximate high-frequency information and its depth.} In principle, increasing depth enhances representational capacity and may facilitate the approximation of complex, high-frequency functions. However, in practice, training very deep networks typically relies on standard residual skip connections to ensure optimization stability. These skip connections, while essential for successful training, may bias learning toward earlier or lower-frequency representations, thereby partially limiting the network’s ability to capture high-frequency information. This observation somewhat aligns with the findings from \citet{zhang2024residualharm}. 

\begin{figure*}[t]
  \centering
  \begin{subfigure}{\figwidththree}
    \centering
    \includegraphics[width=\linewidth]{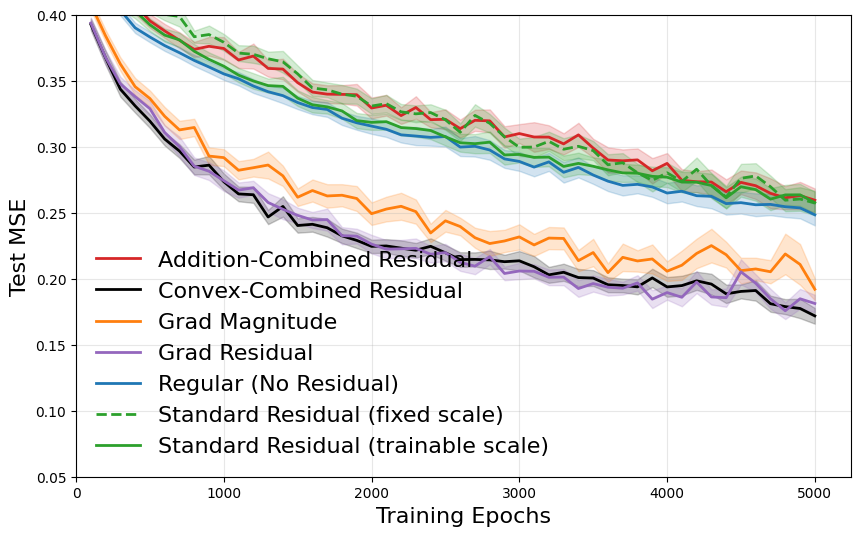}
    \caption{$d=16$}
  \end{subfigure}\hfill
  \begin{subfigure}{\figwidththree}
    \centering
    \includegraphics[width=\linewidth]{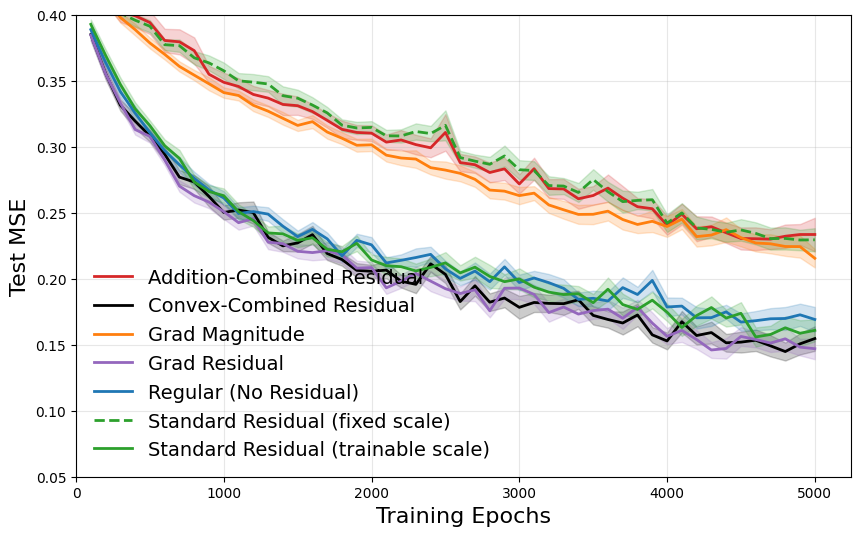}
    \caption{$d=32$}
  \end{subfigure}\hfill
  \begin{subfigure}{\figwidththree}
    \centering
    \includegraphics[width=\linewidth]{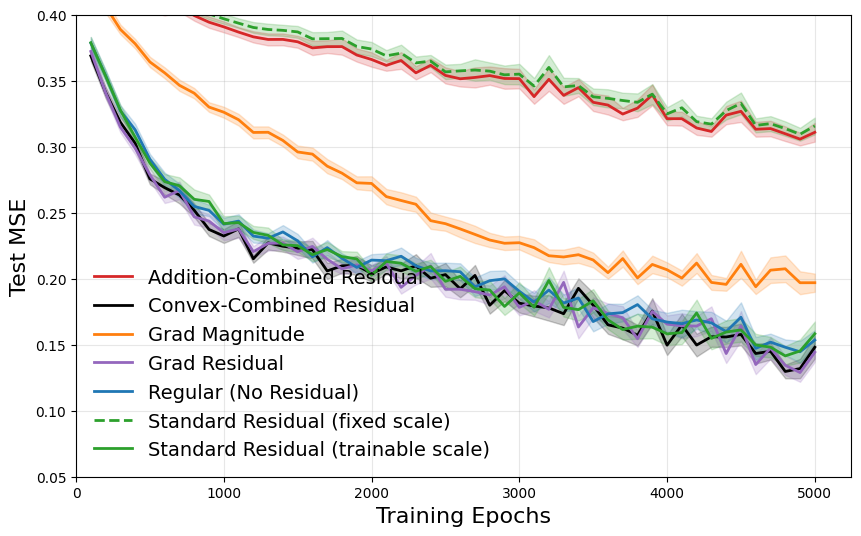}
    \caption{$d=64$}
  \end{subfigure}
  \caption{Learning curves on the sin dataset in terms of test MSE vs.\ number of training epochs. Results are averaged over $30$ random seeds. For standard residual, the solid line is the version with a trainable scalar on the residual connection.}
  \label{fig:sinlc}
\end{figure*}

\emph{What if backpropagating through the gradient.} As a sanity check, we also consider backpropagating through the gradient of $\sum_{i=1}^{d}\nabla \Fvec_i(\xvec)$, for which another gradient can be taken to train the parameters during optimization. We observe that using second-order gradients can lead to slightly better performance, as shown in Figure~\ref{fig:sinlc-sensi}(a). However, we do not adopt this approach due to the additional computational overhead and the potential instability it may introduce for more complicated neural network architectures.

\emph{Normalization or not.} We observe that when the number of hidden units is $16$, using the gradient residual alone can lead to highly unstable learning behavior, as shown in Figure~\ref{fig:sinlc-sensi}(b), which is expected. Anticipating the potential instability of the unnormalized variant, we therefore adopt the normalized gradient design throughout.

\begin{figure}[t]
  \begin{subfigure}{\figwidthfour}
    \centering
    \includegraphics[width=\linewidth]{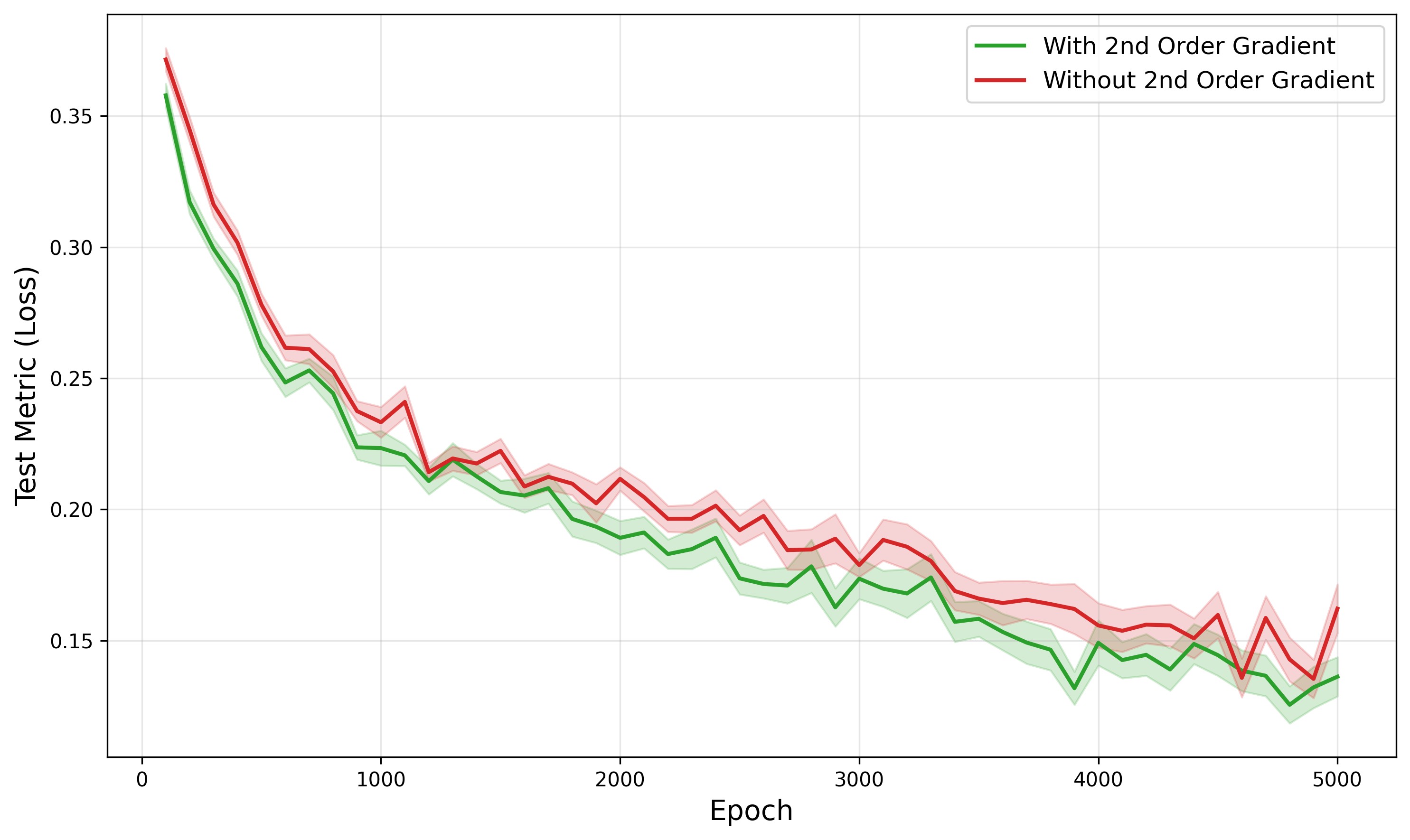}
    \caption{$d=64$}
  \end{subfigure}\hfill
  \begin{subfigure}{\figwidthfour}
    \centering
    \includegraphics[width=\linewidth]{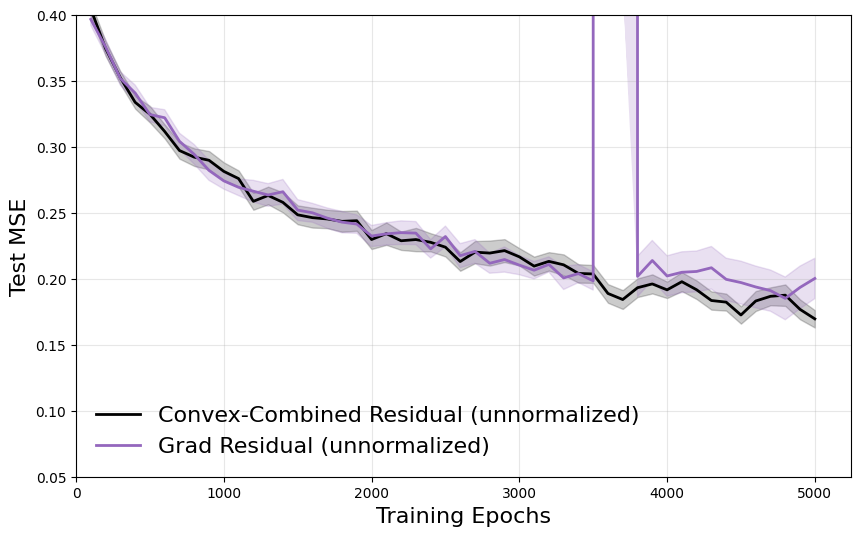}
    \caption{$d=16$}
  \end{subfigure}
  \caption{Learning curves on the sin dataset in terms of test Mean Squared Error (MSE) vs.\ number of training epochs. (a) shows those with/without backpropagating through the gradient. (b) shows those without normalizing the gradient. Results are averaged over $30$ runs.}
  \label{fig:sinlc-sensi}
  \vspace{-0.3cm}
\end{figure}


\emph{Visualizing the true/learned function.} Figure \ref{fig:sin-truevslearned} shows that the solid black curve (\textbf{Convex combined residual}) approximates the sharp spikes much more accurately. 

\begin{figure}[t]
  \centering
  \includegraphics[width=\figwidthtwo]{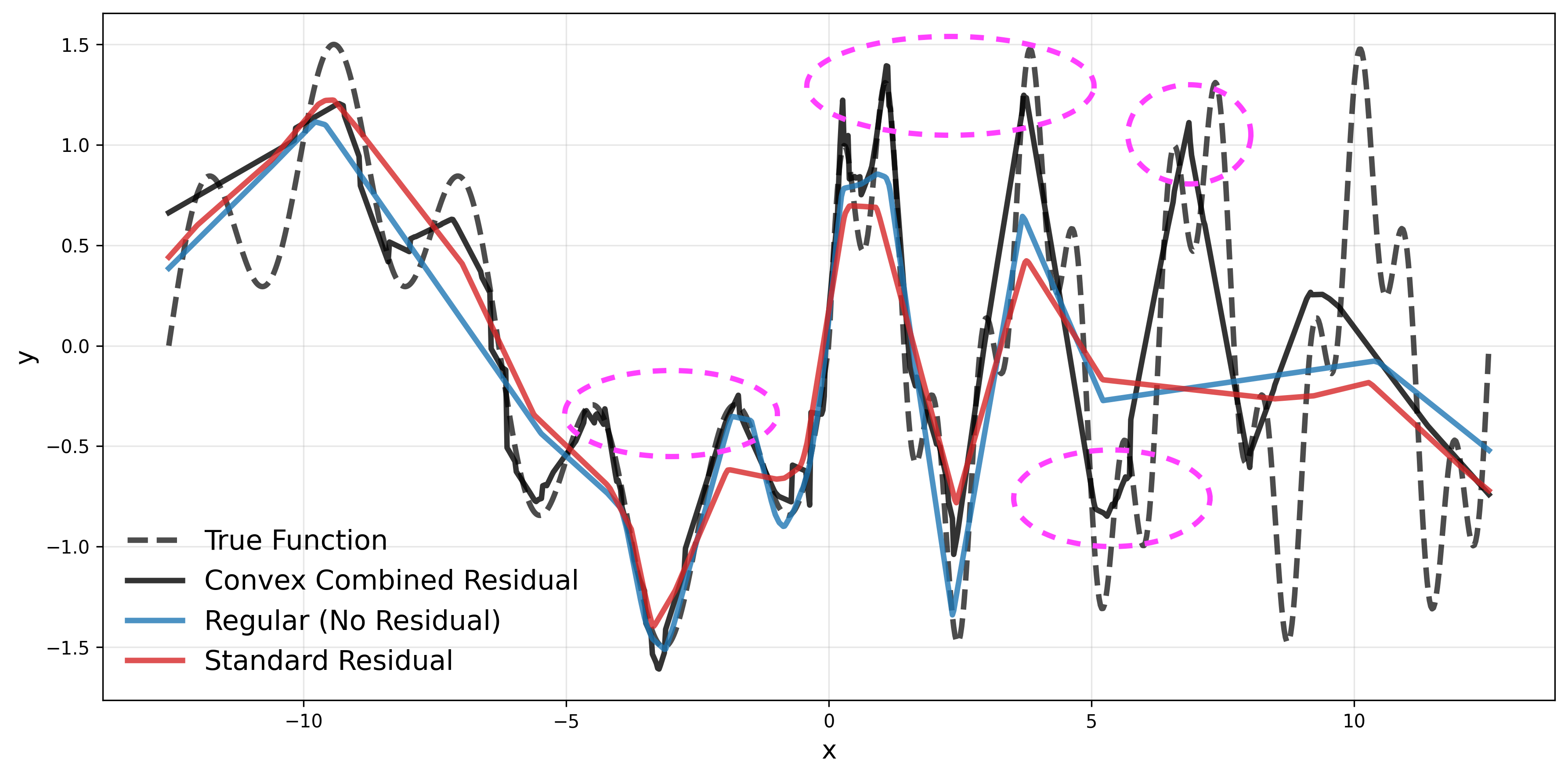}
  \caption{Learned functions vs. the ground-truth function when $d=16$. For each algorithm, we visualize results from the first random seed under the best hyperparameter setting, using models saved at the end of training. The region where our gradient residual method achieves substantially better approximation is highlighted by the \textcolor{purple}{purple circle}. For clarity of visualization, only the most relevant subset of algorithms is included.}
  \label{fig:sin-truevslearned}
  \vspace{-0.3cm}
\end{figure}

\section{Image Experiments}\label{sec:imageexp}
This section presents experiments on the super-resolution task as well as on standard image tasks, including classification and segmentation, to examine the practical utility, strengths, and limitations of the proposed method. 
Note that starting from this section, unless otherwise specified, algorithm names containing \textbf{GradResidual} refer to models using the convexly combined residual connection, which is the \textbf{only difference} from the base model. This design choice allows us to confidently observe and isolate its effect. All missing details and additional empirical results can be found in the Appendix \ref{sec:appendix-srexp} and \ref{sec:appendix-classify-seg}. 

\subsection{Super-resolution task}\label{sec:sr-exp}

We choose super-resolution because it requires learning a mapping from low-resolution inputs to high-resolution outputs, where small perturbations in the input can lead to large changes in fine details in the reconstructed image, which is conceptually similar to approximating a high frequency function. On this task, we aim to address the following questions: 1) are gradient residual connections beneficial on such task where the underlying ground truth might be high frequency; 2) how would different popular architectures or hyperparameter settings affect its performance? 

\textbf{Super-resolution task.} We use the DIV2K \cite{agustsson2017ntire} training dataset and evaluate the testing performance by Peak Signal-to-Noise Ratio (PSNR) on popular benchmark data sets: urban100, rsd100, set14, set5. We adopted three base neural network architectures. 
The first is a \textbf{simplified} version of EDSR with a reduced number of blocks \cite{lim2017edsr} (\textbf{SEDSR vs.\ SEDSR-GradResidual}), and it omits the specialized architectural optimizations typically introduced for super-resolution tasks. This base model provides a fair comparison setting, as it is unclear whether the heavily optimized components in the original EDSR architecture may interact with---or potentially obscure---the effects of gradient residual connections. 
The second more closely matches the original EDSR architecture (we simply name it as \textbf{EDSR vs.\ EDSR-GradResidual}) except that it uses $16$ blocks, which is a popular choice as a baseline. Note that EDSR has approximately $1.36$ million model parameters; in contrast, SEDSR consists of $8$ residual blocks and has approximately $0.6$ million trainable parameters. 
The third follows the architecture proposed in \cite{ledig2017srgan} (\textbf{SRResNN vs. SRResNN-GradResidual}). 
We do not devote substantial effort to exactly matching every architectural or implementation detail of these models to the original papers; however, we ensure that the only difference between each pair of algorithms is the replacement of the original residual connection with our gradient residual in each residual block, while keeping the global residual unchanged if any.

\textbf{SEDSR/EDSR-based results}. Table~\ref{tab:sedsr_edsr_final_metrics} reports the results in Peak Signal-to-Noise Ratio (PSNR) on different test sets when using the SEDSR and EDSR as base models. 
Figure~\ref{fig:sedsr} presents the corresponding learning curves for the SEDSR-based models, demonstrating a statistically significant advantage of the gradient residual connection. 

However, it should be noted that gradient residual connections typically incur higher computational cost. For example, on SEDSR, the standard model requires approximately $10.5$ ms per update, whereas the gradient-based variant requires about $23.7$ ms per update. Despite this increase in wall-clock time per update, the two approaches share the same asymptotic computational complexity. Notably, the smaller SEDSR-GradResidual model outperforms EDSR despite using nearly half the number of parameters, indicating that gradient residual connections can provide stronger fitting capacity with the same model size. Furthermore, scaling up the model size in EDSR-GradResidual does not yield additional gains over the SEDSR variant, which point to more effective utilization of model capacity. As data availability and computational resources continue to grow, such efficiency may become increasingly valuable, making gradient residual connections an attractive design choice for scalable architectures.

Another observation is that for SEDSR-GradResidual, the best-performing convex combination factor $\alpha$ is consistently selected as $3$, indicating a strong emphasis on the gradient representation. In contrast, when EDSR is used as the base model, the method selects $\alpha = -3$ as the optimal initialization for the convex combination factor. These results indicate that the gradient representation remains effective, although its relative importance may decrease as model capacity increases. 

\emph{The effect of mini-batch size.} Since gradient computation can be sensitive to noise, it is reasonable to conjecture that increasing the mini-batch size may amplify the benefits of gradient residual connections. Figure~\ref{fig:sedsr-batchsize} in the appendix shows that the advantages of gradient residual connections become more pronounced as the batch size increases, which is consistent with our conjecture. 

\begin{table}[ht]
\centering
\begin{tabular}{lcc}
\toprule
Dataset & SEDSR & SEDSR-GradResidual \\
\midrule
Urban100 & $29.81 \pm 0.01089$ & $\mathbf{30.06 \pm 0.008178}$ \\
BSD100   & $31.63 \pm 0.003525$ & $\mathbf{31.73 \pm 0.006626}$ \\
Set5     & $37.11 \pm 0.01430$ & $\mathbf{37.23 \pm 0.02093}$ \\
Set14    & $32.84 \pm 0.007602$ & $\mathbf{32.96 \pm 0.006835}$ \\
DIV2K    & $35.17 \pm 0.004458$ & $\mathbf{35.31 \pm 0.008841}$ \\
\bottomrule
\end{tabular}
\begin{tabular}{lcc}
\toprule
Dataset & EDSR & EDSR-GradResidual \\
\midrule
Urban100 & $29.93 \pm 0.03905$ & $\mathbf{29.99 \pm 0.006349}$ \\
BSD100   & $31.68 \pm 0.01952$ & $\mathbf{31.70 \pm 0.004340}$ \\
Set5     & $37.16 \pm 0.01827$ & $\mathbf{37.19 \pm 0.01344}$ \\
Set14    & $32.88 \pm 0.01642$ & $\mathbf{32.91 \pm 0.007583}$ \\
DIV2K    & $35.22 \pm 0.03180$ & $\mathbf{35.25 \pm 0.01306}$ \\
\bottomrule
\end{tabular}
\caption{Average final PSNR $\pm$ standard error. \emph{Final PSNR} is defined as averaging over the final $25\%$ evaluations before averaging over $3$ random seeds, hence it has a smaller standard error than the best metric (as shown in Appendix \ref{sec:appendix-srexp}). For all algorithms, the selected learning rate is $0.0003$ and SEDSR-GradResidual chooses $\alpha=3.0$. However, for the EDSR-GradResidual results, $\alpha=-3.0$ is chosen. 
}\label{tab:sedsr_edsr_final_metrics}
\vspace{-0.3cm}
\end{table}

\begin{figure*}[t]
\centering
\begin{subfigure}{\figwidthfive}
\centering
\includegraphics[width=\linewidth]{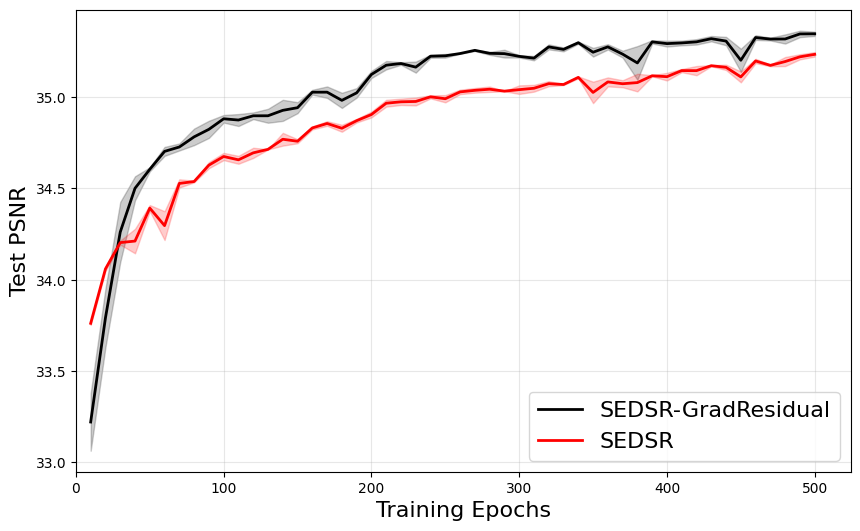}
\caption{DIV2K (validation)}
\end{subfigure}\hfill
\begin{subfigure}{\figwidthfive}
\centering
\includegraphics[width=\linewidth]{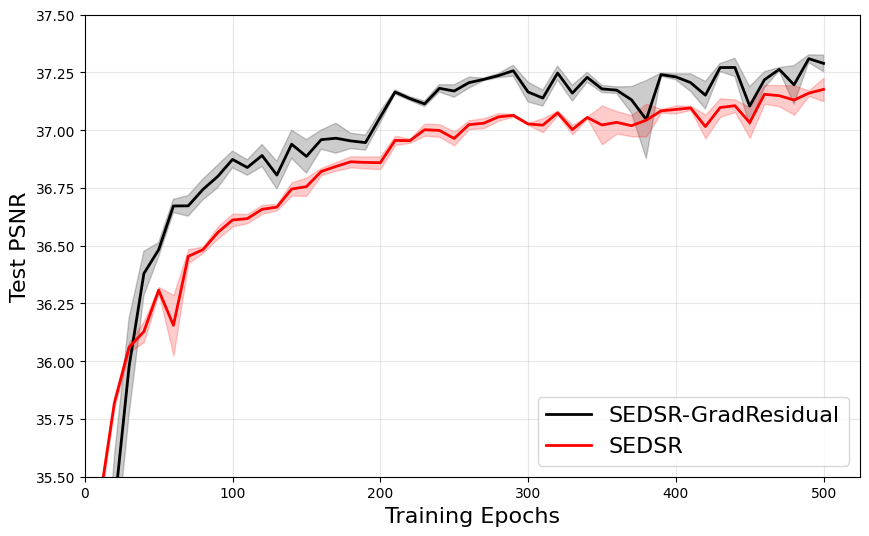}
\caption{Set5}
\end{subfigure}\hfill
\begin{subfigure}{\figwidthfive}
\centering
\includegraphics[width=\linewidth]{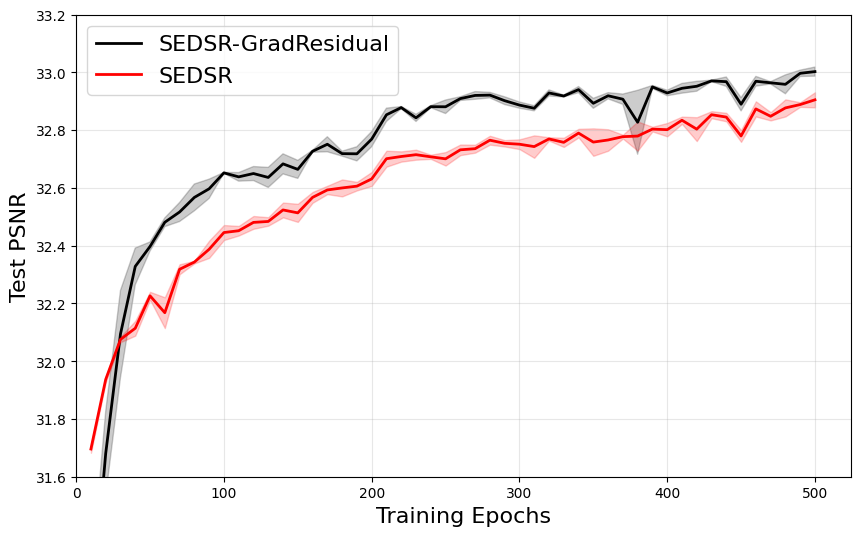}
\caption{Set14}
\end{subfigure}
\begin{subfigure}{\figwidthfive}
\centering
\includegraphics[width=\linewidth]{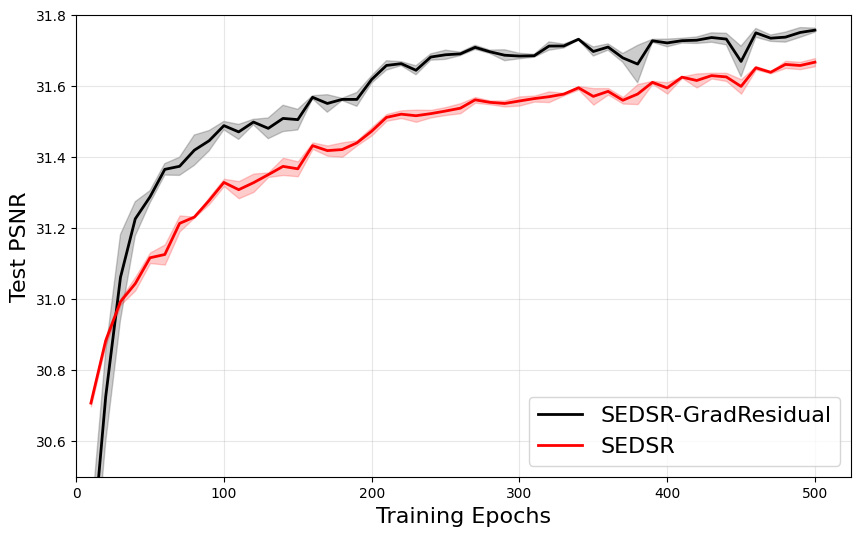}
\caption{BSD100}
\end{subfigure}\hfill
\begin{subfigure}{\figwidthfive}
\centering
\includegraphics[width=\linewidth]{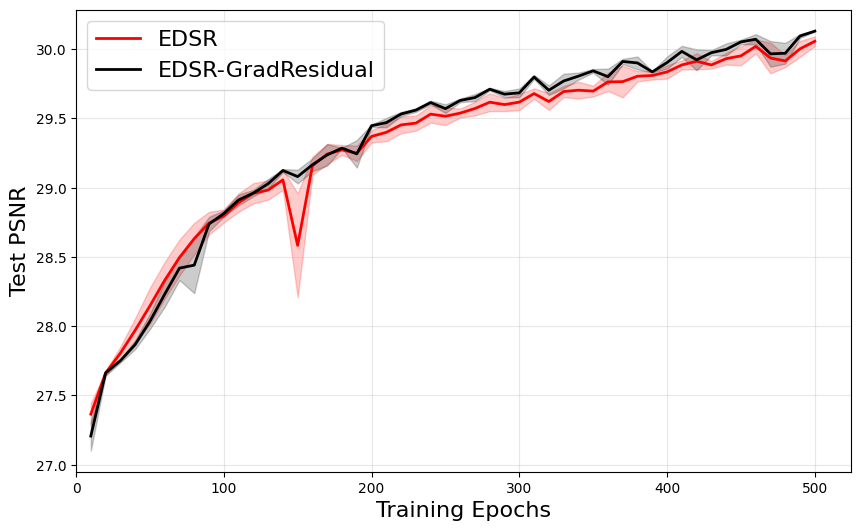}
\caption{Urban100}
\end{subfigure}

\caption{Learning curves of test performance measured by mean PSNR (mean over images) vs. training epochs on standard benchmarks. We train for 500 epochs in total and evaluate every 10 epochs. The results are averaged over 3 random seeds. 
}
\label{fig:sedsr}
\end{figure*}

SRResNet-based results. Table~\ref{tab:srresnet_final} and \ref{tab:srresnet_best} in Appendix \ref{sec:appendix-srexp} report the results. We observe that: (1) there is no significant difference between the gradient-residual variant and the baseline; (2) this architecture appears to be less effective than the previously evaluated SEDSR/EDSR models for both variants; and (3) batch normalization appears to be detrimental, both with and without gradient residuals.

Additional experiments. We also present two additional sets of experiments conducted on the three base architectures.
(1) Instead of using a convex combination, for the gradient residual variant we introduce two independent trainable scalars---one for the standard residual connection and one for the gradient residual. Correspondingly (see \eqref{eq:independent-resscale} and \eqref{eq:independent-resscale} in Appendix \ref{sec:appendix-variants-sr}), we introduce a single trainable scalar for the baseline standard residual model.
(2) Building on (1), we further introduce two additional layers to parameterize the scalar weights for the standard residual and the gradient residual, respectively, making the scaling sample-dependent (see \eqref{eq:sample-dependent-gradscale} and \eqref{eq:sample-dependent-resscale} in Appendix \ref{sec:appendix-variants-sr}).

For (1), we observe mostly no statistically significant differences between the baseline and GradResidual variants across all three base architectures. For (2), when using SEDSR and EDSR as base models, we again do not observe a significant difference between the gradient residual and standard residual variants. However, when SRResNet is used as the base model, we observe a significant improvement from employing the gradient residual over the standard residual, as shown in Table \ref{tab:srresnet_samplebasedscalar_final}. These mixed results may indicate that the underlying network architecture plays a role in the effectiveness of gradient residuals, as those existing models are typically optimized when standard residual connections are used.

\begin{table}[ht]
\centering
\vspace{-0.3cm}
\begin{tabular}{lcc}
\toprule
Dataset & SRResNet & SRResNet-GradResidual \\
\midrule
Urban100 & $27.58 \pm 0.09445$ & $\mathbf{28.27 \pm 0.1136}$ \\
BSD100   & $30.40 \pm 0.1267$  & $\mathbf{30.84 \pm 0.04787}$ \\
Set5     & $34.28 \pm 0.4258$  & $\mathbf{35.57 \pm 0.1156}$ \\
Set14    & $31.25 \pm 0.1464$  & $\mathbf{31.53 \pm 0.03792}$ \\
DIV2K    & $33.16 \pm 0.2742$  & $\mathbf{33.80 \pm 0.04704}$ \\
\bottomrule
\end{tabular}
\caption{Final PSNR for SRResNet with sample-based scaling. The results are averaged over 3 random seeds.}
\label{tab:srresnet_samplebasedscalar_final}
\vspace{-0.3cm}
\end{table}

\subsection{Classification and Segmentation}\label{sec:classify-seg}

In general, we do not observe significant performance differences between models with and without gradient residual connections on standard image classification or semantic segmentation tasks. This behavior is intuitively reasonable, as such tasks are primarily driven by low-frequency semantic information that determines category-level decisions. Consequently, the underlying functions being approximated are relatively smooth and are less dominated by high-frequency components, reducing the potential benefit of explicitly enhancing gradient-based residual pathways. Additional details and results are in Appendix \ref{sec:appendix-classify-seg}. 

\section{Discussions}\label{sec:discussion}
This work proposes gradient-based residual connection that leverages gradient vectors computed from later-layer representations with respect to earlier-layer representations, with the goal of enhancing a neural network’s ability to approximate high-frequency target function. We provide intuitive theoretical insights into why such gradient information can be effective for separating data points in high-frequency regions. We conduct a series of controlled synthetic experiments to systematically investigate implementation choices and validate the effectiveness of the proposed approach. In addition, we demonstrate the practical utility and different implementation choices of gradient residual connections on super-resolution tasks, where some variants yield significant performance improvements, and show that the method remains competitive on conventional image classification and semantic segmentation tasks.

\textbf{Limitations and Future Work}. This work has limited exploration of their interactions with other training strategies, such as optimization algorithms, or regularization techniques that aim to reduce input gradients in order to improve robustness \cite{ross2017inputgradreg,drucker1992doublebackprop}. In addition, our experiments do not consider scenarios in which the training and testing data may exhibit significantly different frequency characteristics. While this direction is potentially interesting, many domain shift problems are primarily concerned with changes in data distributions rather than explicit frequency shifts. Finally, our focus is on isolating and evaluating the utility of gradient residual connections relative to conventional residual connections. Hence, we did not explore architectural optimizations specifically tailored to gradient residual. In particular, we did not investigate where such connections should originate or terminate, nor how they interact with existing architectural components such as upsampling and downsampling operations, global skip connections, or other popular designs. Since most existing neural network architectures have been extensively optimized for standard residual connections, they may be suboptimal when directly adapted to gradient-based residuals. A systematic exploration of architectures better suited to gradient residual connections would likely require substantial engineering effort and is beyond the scope of the present study.

\newpage
\section*{Impact Statement}
This paper presents work aimed at advancing the field of Machine Learning. There are many potential societal consequences of our work, none
which we feel must be specifically highlighted here.
\bibliography{ref}
\bibliographystyle{icml2026}

\newpage
\appendix
\onecolumn
\section{Appendix}
The appendix includes the following contents. 

\begin{enumerate}
    \item Section \ref{sec:appendix-theory}: Proof for Theorem \ref{thm:main-thm} and its associated lemmas.
    \item Section \ref{sec:appendix-synthetic}: Additional details for Section \ref{sec:sinexp}.
    \item Section \ref{sec:appendix-srexp} and \ref{sec:appendix-variants-sr}: Additional details and results for super resolution task from Section \ref{sec:sr-exp}. 
    \item Section \ref{sec:appendix-classify-seg} Additional details for Section \ref{sec:classify-seg}. 
    \item Section \ref{sec:appendix-relatedwork}: Additional literature review on the connections between the frequency perspective and machine learning, as well as additional work on residual neural networks. 
\end{enumerate}

\subsection{Proof}\label{sec:appendix-theory}

\gdboundtwopoints*

\begin{proof}
To show \eqref{eq:den_lb_x0}
\begin{align}
   \| \nabla f(\xvec_0) \| &= \| \nabla f_h(\xvec_0) + \nabla f_l(\xvec_0) \| \\
    &\ge \|\nabla f_h(\xvec_0)\| - \|\nabla f_l(\xvec_0)\| \ge (1-\varepsilon)L, 
\end{align}
where the first inequality is a basic property of norm and the second one is due to $\|\nabla f_h(\xvec_0)\| \ge L$ by definition and $\|\nabla f_l(\xvec_0)\|\le \varepsilon L$ by Assumption \ref{assm:boundedgd}. 

The gradient of $f_h$ at $\xvec_1$ satisfies:
\begin{align}
\|\nabla f_h(\xvec_1)\|
\ge \|\nabla f_h(\xvec_0)\|-\|\nabla f_h(\xvec_0)+\nabla f_h(\xvec_1)\|
\ge L - C_d\frac{\delta}{\omega}M.
\end{align}
Then,
\begin{align}
\|\nabla f(\xvec_1)\| &=\|\nabla f_h(\xvec_1)+\nabla f_l(\xvec_1)\| \\
&\ge \|\nabla f_h(\xvec_1)\|-\|\nabla f_l(\xvec_1)\| \\
&\ge \left(L - C_d\frac{\delta}{\omega}M\right)-\varepsilon L \\
&=(1-\varepsilon)L - C_d\frac{\delta}{\omega}M.
\end{align}
\end{proof}

\highgdsumbound*

\begin{proof}

By definition, $\xvec_1=\xvec_0+\frac{\pi}{\omega}\uvec$. Using \eqref{eq:grad_fourier},
\begin{align}
&\nabla f_h(\xvec_0)+\nabla f_h(\xvec_1)\\
&=\frac{1}{(2\pi)^d}\int_{B(\omega\uvec,\delta)} i\xivec\,\fhat_h(\xivec)
\left(e^{i\xvec_0\cdot\xivec} + e^{i\xvec_1\cdot\xivec}\right)\,d\xivec \nonumber\\
&=\frac{1}{(2\pi)^d}\int_{B(\omega\uvec,\delta)} i\xivec\,\fhat_h(\xivec)\,e^{i\xvec_0\cdot\xivec}
\left(1 + e^{i(\pi/\omega)\uvec\cdot\xivec}\right)\,d\xivec.
\label{eq:sum_high_grad}
\end{align}

For $\xivec\in B(\omega\uvec,\delta)$, write $\xivec=\omega\uvec+\rvec$ with $\|\rvec\|\le \delta$. Then
\begin{align}
(\pi/\omega)\uvec\cdot\xivec
&=(\pi/\omega)\uvec\cdot(\omega\uvec+\rvec) =\pi + m(\xivec), \\
\text{where} \quad m(\xivec) &\defeq (\pi/\omega)\uvec\cdot\rvec, \\
|m(\xivec)| &\le \pi\delta/\omega.
\end{align}
Hence
\[
1 + e^{i(\pi/\omega)\uvec\cdot\xivec}
=1+e^{i(\pi+m(\xivec))}
=1-e^{im(\xivec)},
\]
as $e^{i x} = cos (x) + i sin (x)$ by Euler's formulae. 

Using $|1-e^{it}|\le |t|$ for all $t\in\RR$,
\[
\left|1 + e^{i(\pi/\omega)\uvec\cdot\xivec}\right|
=\left|1-e^{im(\xivec)}\right|
\le |m(\xivec)|
\le \pi\delta/\omega.
\]
Taking norm of \eqref{eq:sum_high_grad} and applying the above,
\begin{align}
& \|\nabla f_h(\xvec_0)+\nabla f_h(\xvec_1)\| \\
&\le \frac{1}{(2\pi)^d}\int_{B(\omega\uvec,\delta)} \|\xivec\|\,|\fhat_h(\xivec)|
\left|1 + e^{i(\pi/\omega)\uvec\cdot\xivec}\right|\,d\xivec \nonumber\\
&\le \frac{1}{(2\pi)^d}\frac{\pi\delta}{\omega}\int_{B(\omega\uvec,\delta)} \|\xivec\|\,|\fhat_h(\xivec)|\,d\xivec.
\end{align}
Since $|\fhat_h|\le |\fhat|$ and $\fhat_h$ is supported on $B(\omega\uvec,\delta)$, the last integral is bounded by $M$ in \eqref{eq:def_M_Cd}. Therefore
\begin{equation}
\|\nabla f_h(\xvec_0)+\nabla f_h(\xvec_1)\|
\le C_d\frac{\delta}{\omega}M.
\end{equation}

\end{proof}

\maintheorem*

\begin{proof}
We use the Fourier conventions stated in the theorem:
\[
\fhat(\xivec)
=\int_{\RR^d} f(\xvec)e^{-i\xvec\cdot\xivec}\,d\xvec,
\quad
f(\xvec)
=\frac{1}{(2\pi)^d}\int_{\RR^d}\fhat(\xivec)e^{i\xvec\cdot\xivec}\,d\xivec.
\]
By \eqref{eq:assump_moment1} and the definition of $f_h$, we may differentiate under the integral on the band to obtain
\begin{equation}\label{eq:grad_fourier}
\nabla f_h(\xvec)
=\frac{1}{(2\pi)^d}\int_{B(\omega\uvec,\delta)} i\xivec\,\fhat_h(\xivec)\,e^{i\xvec\cdot\xivec}\,d\xivec.
\end{equation}

\paragraph{Bound the sum of gradient in high frequency components.}
By Lemma \ref{lemma:eq:high_sum_bound}, we know
\[
\|\nabla f_h(\xvec_0)+\nabla f_h(\xvec_1)\|
\le C_d\frac{\delta}{\omega}M.
\]

\paragraph{Include the full gradient.}
Since $f=f_l+f_h$,
\[
\nabla f(\xvec_0)+\nabla f(\xvec_1)
=\big(\nabla f_h(\xvec_0)+\nabla f_h(\xvec_1)\big)
+\big(\nabla f_l(\xvec_0)+\nabla f_l(\xvec_1)\big).
\]
Hence, by the triangle inequality and \eqref{eq:assump_lowgrad},
\begin{align}\label{eq:full_sum_bound}
& \|\nabla f(\xvec_0)+\nabla f(\xvec_1)\| \\
& \le \|\nabla f_h(\xvec_0)+\nabla f_h(\xvec_1)\| + \|\nabla f_l(\xvec_0)\|+\|\nabla f_l(\xvec_1)\| \\
&\le C_d\frac{\delta}{\omega}M + 2\varepsilon L.
\end{align}

\paragraph{Normalization and the bound on $\|\gvec_0+\gvec_1\|$.}
We use the elementary inequality valid for any nonzero vectors $\avec,\bvec$:
\begin{equation}\label{eq:unit_sum_ineq}
\left\|\frac{\avec}{\|\avec\|}+\frac{\bvec}{\|\bvec\|}\right\|
\le \frac{2\|\avec+\bvec\|}{\min(\|\avec\|,\|\bvec\|)}.
\end{equation}
Applying \eqref{eq:unit_sum_ineq} with $\avec=\nabla f(\xvec_0)$ and $\bvec=\nabla f(\xvec_1)$ gives
\begin{equation}\label{eq:g_sum_via_full_sum}
\|\gvec_0+\gvec_1\|
\le \frac{2\|\nabla f(\xvec_0)+\nabla f(\xvec_1)\|}
{\min(\|\nabla f(\xvec_0)\|,\|\nabla f(\xvec_1)\|)}.
\end{equation}

By Lemma \ref{lemma:eq:den_lb_x1x0}, we can bound the denominator
\begin{equation}\label{eq:denominator-min-bound}
\min(\|\nabla f(\xvec_0)\|,\|\nabla f(\xvec_1)\|)\ge (1-\varepsilon)L - C_d\frac{\delta}{\omega}M, 
\end{equation}
since $\|\nabla f(\xvec_0)\| \ge (1-\varepsilon)L$ and $\|\nabla f(\xvec_1)\| \ge (1-\varepsilon)L - C_d\frac{\delta}{\omega}M$. 

Combining \eqref{eq:full_sum_bound}, \eqref{eq:g_sum_via_full_sum}, and \eqref{eq:denominator-min-bound}, we obtain
\[
\|\gvec_0+\gvec_1\|
\le
\frac{2\Big(C_d\frac{\delta}{\omega}M + 2\varepsilon L\Big)}
{(1-\varepsilon)L - C_d\frac{\delta}{\omega}M},
\]
which is \eqref{eq:main_bound}.

\paragraph{Asymptotics and angle consequence.}
From \eqref{eq:main_bound}, if $\varepsilon\to 0$ and $\frac{\delta}{\omega} \to 0$ (with $M/L$ controlled), then the right-hand side is
$O\!\left(\frac{\delta}{\omega}\frac{M}{L}+\varepsilon\right)$, hence $\|\gvec_0+\gvec_1\|=o(1)$.

Finally, since $\gvec_0,\gvec_1$ are unit vectors and $\theta=\angle(\gvec_0,\gvec_1)\in[0,\pi]$,
\[
\|\gvec_0+\gvec_1\| = 2\cos(\theta/2).
\]
Thus if the right-hand side of \eqref{eq:main_bound} is $\le 2\sin(\eta/2)$, then $\cos(\theta/2)\le \sin(\eta/2)=\cos((\pi-\eta)/2)$, which implies $\theta\ge \pi-\eta$.
\end{proof}

\subsection{Synthetic experiments}\label{sec:appendix-synthetic}

For all the experiments on the synthetic sin datasets, the learning rate is swept over $\{0.125, 0.03125, 0.0078125, 0.001953125\}$ for all algorithms. For the standard residual implementation, the trainable scalar $\alpha$ in $\mathbf{F}(\mathbf{x}) + \alpha \mathbf{x}$ is initialized to $1.0 / \sqrt{d}$, where $d$ denotes the network width. This heuristic is motivated by the observation that, on this dataset, wider neural networks tend to be more adversely affected by residual connections. Below we supplement one more results on studying the sensitivity w.r.t. the convex combination scalar (i.e., $\alpha$) initialization. 

\emph{Sensitivity w.r.t.\ scalar initialization.}
To further validate the utility of the convex-combination design, we visualize in Figure~\ref{fig:sinlc-init} how performance changes under different initializations of the scalar $\alpha$ in the convex combination \eqref{eq-convexgdresidual}. Since the scalar is passed through a sigmoid, the values $-3$ and $3$ are already extreme, corresponding to regions where the sigmoid is close to saturation. Initializing at $3$ places most weight on the gradient residual branch, and the resulting performance remains strong, which further supports the usefulness of the gradient representation. Practically, this also simplifies hyperparameter selection: it is often sufficient to sweep only two values such as $-3$ and $3$ indicating potential inductive bias towards the low frequency region or high frequency region given the task, or to choose based on prior knowledge---for instance, using $2$ or $3$ when the target function is expected to be high-frequency. 

\begin{figure}[t]
  \centering
  \includegraphics[width=\figwidthtwo]{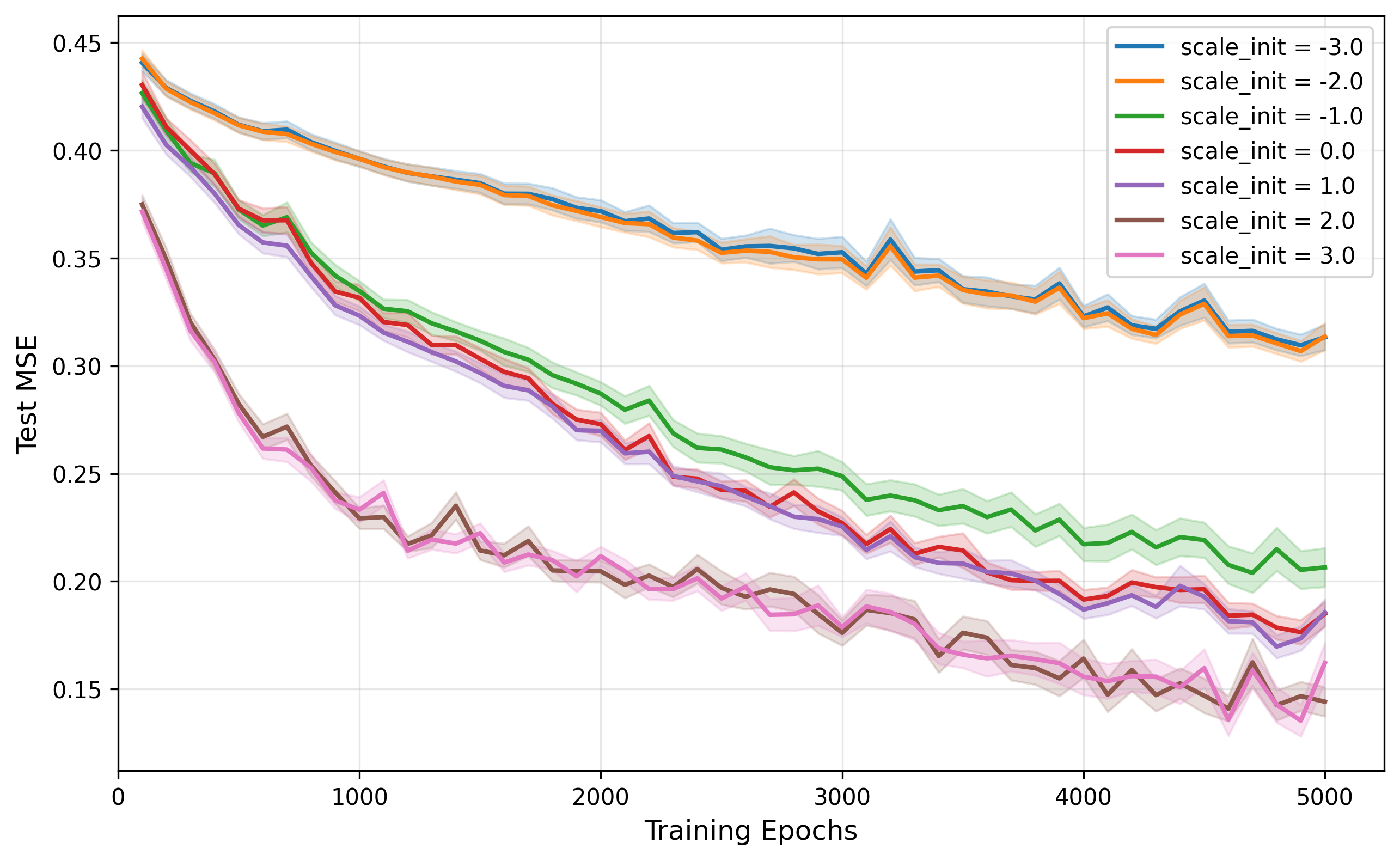}
  \caption{Learning curves (testing MSE v.s. training epochs) when using different initialization scalars for our approach. Results are averaged over 30 runs/random seeds.}
  \label{fig:sinlc-init}
\end{figure}

\subsection{Super resolution experiments}\label{sec:appendix-srexp}

\textbf{Dataset and basic setup.} We use the DIV2K dataset for $\times 2$ single-image super-resolution (SR). Each sample consists of a low-resolution (LR) RGB image tensor and its aligned high-resolution (HR) target of size $3 \times (2H) \times (2W)$, where $H$ and $W$ vary across images. Pixel values are converted to tensors in $[0,1]$. During training, we extract aligned random crops: we use $48 \times 48$ patches from the LR image and the corresponding $96 \times 96$ patches from the HR image.The HR patch size is scaled by the upsampling factor. We train using mean squared error (MSE) loss between the predicted and ground-truth HR images. 

For validation and benchmark evaluation, images are processed at full resolution (i.e. without cropping), with standard \emph{modcrop} alignment to ensure that the image dimensions are divisible by the scale factor. We additionally evaluate on standard SR benchmark datasets, including Set5, Set14, BSD100, and Urban100, when available. Performance is reported using peak signal-to-noise ratio (PSNR, in dB), computed on the luminance (Y) channel (MATLAB-style RGB$\rightarrow$Y conversion), with border shaving equal to the scale factor, following common SR practice. For all algorithms, we sweep over learning rates in ${3\times10^{-4}, 1\times10^{-4}}$ using the Adam optimizer. We employ a step decay learning rate schedule, halving the learning rate once (scaling by $0.5$) at epoch $200$ over a total of $500$ training epochs. For gradient-based residual, the initialization of $\alpha$ is from $\{-3, 3\}$. When reading results, one should be aware that \textbf{final evaluation (i.e. average final PSNR)} means averaged over $25\%$ evaluations first before averaging over random seeds, while \textbf{Peak/best evaluation} means reporting the best evaluation during training after averaging over random seeds. Training batch size is $32$ unless otherwise specified. Three random seeds are used on all SR tasks. 

\textbf{Additional results supplementary to Table \ref{tab:sedsr_edsr_final_metrics}.} Tables \ref{tab:sedsr_best_eval_metrics} and \ref{tab:edsr_best_eval_metrics} show the best evaluation results during the training stage, as a supplement to the final evaluation (\ref{tab:sedsr_edsr_final_metrics}) shown in the main body. 

\begin{table}[ht]
\centering
\begin{tabular}{lcc}
\toprule
Dataset & SEDSR & SEDSR-GradResidual \\
\midrule
Urban100 & $29.91 \pm 0.03592$  & $\mathbf{30.14 \pm 0.02891}$ \\
BSD100   & $31.67 \pm 0.01146$  & $\mathbf{31.76 \pm 0.005760}$ \\
Set5     & $37.18 \pm 0.05099$  & $\mathbf{37.31 \pm 0.01679}$ \\
Set14    & $32.91 \pm 0.02619$  & $\mathbf{33.00 \pm 0.01582}$  \\
DIV2K    & $35.23 \pm 0.01241$  & $\mathbf{35.35 \pm 0.01184}$  \\
\bottomrule
\end{tabular}
\caption{Best PSNR $\pm$ std.\ error. Best PSNR is recorded during the training stage, after averaging over 3 random seeds.}
\label{tab:sedsr_best_eval_metrics}
\end{table}

\begin{table}[ht]
\centering
\begin{tabular}{lcc}
\toprule
Dataset & EDSR & EDSR-{\textbf{GradResidual}} \\
\midrule
Urban100 & $30.06 \pm 0.03545$ & $\mathbf{30.13 \pm 0.002999}$ \\
BSD100   & $31.73 \pm 0.02244$ & $\mathbf{31.76 \pm 0.008334}$ \\
Set5     & $37.27 \pm 0.04251$ & $\mathbf{37.29 \pm 0.01640}$ \\
Set14    & $32.97 \pm 0.02107$ & $\mathbf{33.00 \pm 0.01405}$ \\
DIV2K    & $35.32 \pm 0.03509$ & $\mathbf{35.36 \pm 0.01469}$ \\
\bottomrule
\end{tabular}
\caption{Best PSNR $\pm$ std.\ error. Number of random seeds = 3. For EDSR-GradResidual, the best $\alpha=-3.0$.}
\label{tab:edsr_best_eval_metrics}
\end{table}

\textbf{Additional results on comparing the effect of mini-batch size.} Figure \ref{fig:sedsr-batchsize} shows the learning curves when using mini-batch size $=16$ and $32$ when using SEDSR as a base architecture. As one can see that when batch size is $32$, the advantage of gradient residuals look more obvious. 

\begin{figure*}[t]
\centering
\begin{subfigure}{\figwidthfive}
\centering
\includegraphics[width=\linewidth]{figures/figures_sr_all/srnetsimple_PSNR_div2k_32_lc.png}
\caption{DIV2K}
\end{subfigure}\hfill
\begin{subfigure}{\figwidthfive}
\centering
\includegraphics[width=\linewidth]{figures/figures_sr_all/srnetsimple_PSNR_set5_div2k_32_lc.png}
\caption{Set5}
\end{subfigure}\hfill
\begin{subfigure}{\figwidthfive}
\centering
\includegraphics[width=\linewidth]{figures/figures_sr_all/srnetsimple_PSNR_set14_div2k_32_lc.png}
\caption{Set14}
\end{subfigure}
\begin{subfigure}{\figwidthfive}
\centering
\includegraphics[width=\linewidth]{figures/figures_sr_all/srnetsimple_PSNR_bsd100_div2k_32_lc.png}
\caption{BSD100}
\end{subfigure}\hfill
\begin{subfigure}{\figwidthfive}
\centering
\includegraphics[width=\linewidth]{figures/figures_sr_all/edsr_PSNR_urban100_div2k_32_lc.png}
\caption{Urban100}
\end{subfigure}

\begin{subfigure}{\figwidthfive}
\centering
\includegraphics[width=\linewidth]{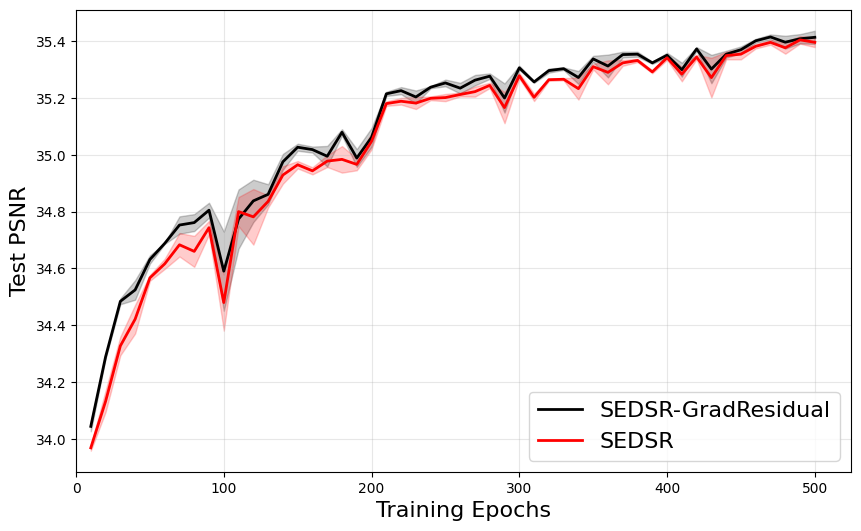}
\caption{DIV2K}
\end{subfigure}\hfill
\begin{subfigure}{\figwidthfive}
\centering
\includegraphics[width=\linewidth]{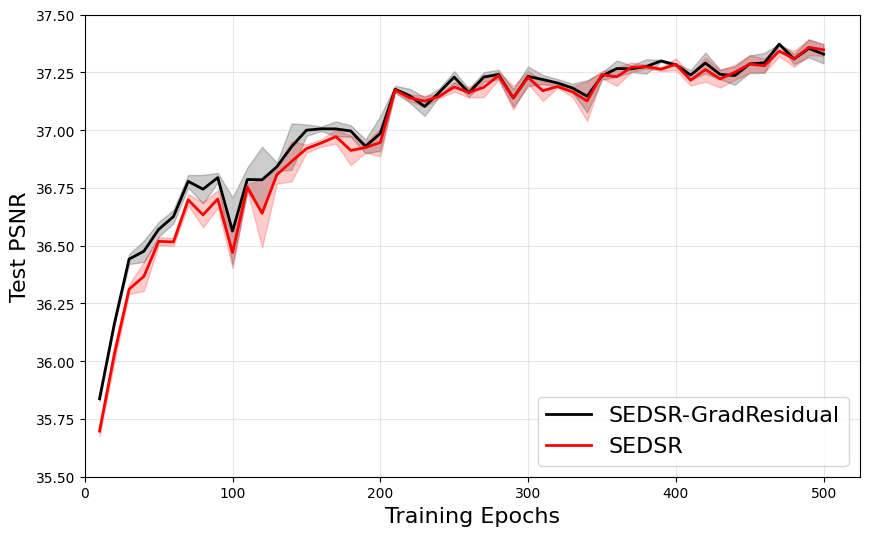}
\caption{Set5}
\end{subfigure}\hfill
\begin{subfigure}{\figwidthfive}
\centering
\includegraphics[width=\linewidth]{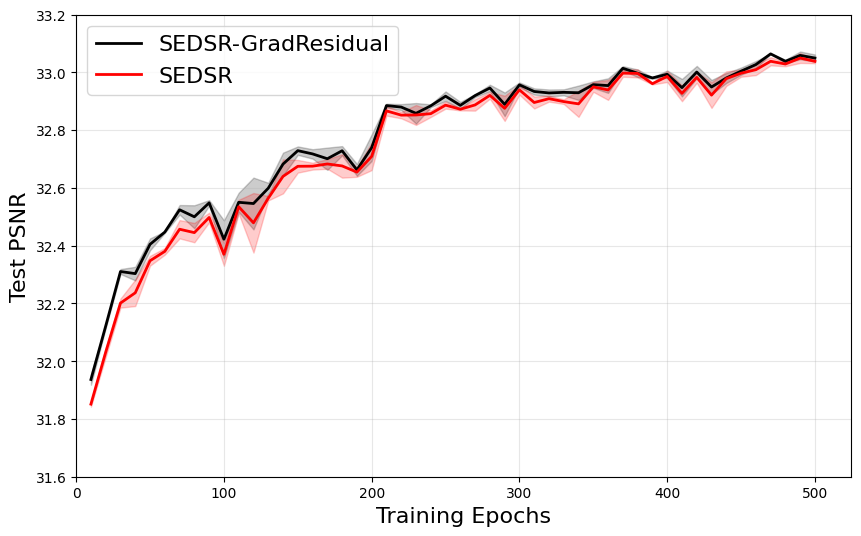}
\caption{Set14}
\end{subfigure}
\begin{subfigure}{\figwidthfive}
\centering
\includegraphics[width=\linewidth]{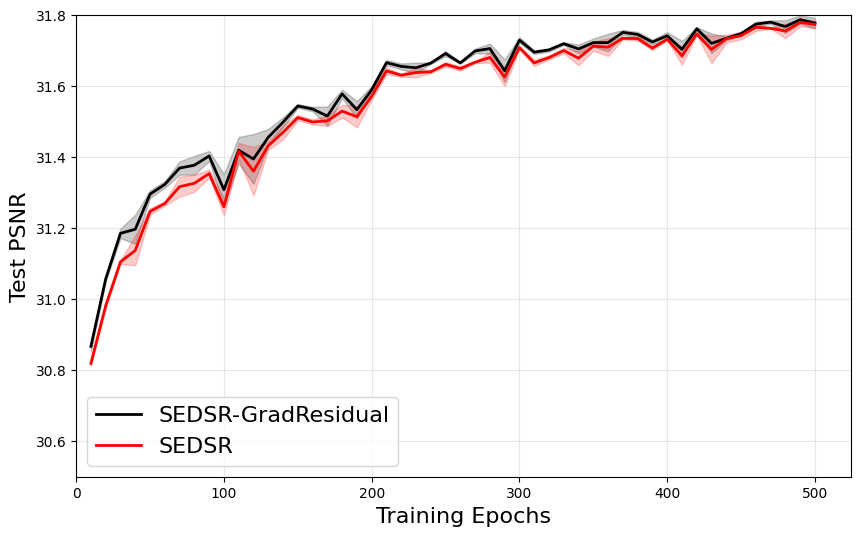}
\caption{BSD100}
\end{subfigure}\hfill
\begin{subfigure}{\figwidthfive}
\centering
\includegraphics[width=\linewidth]{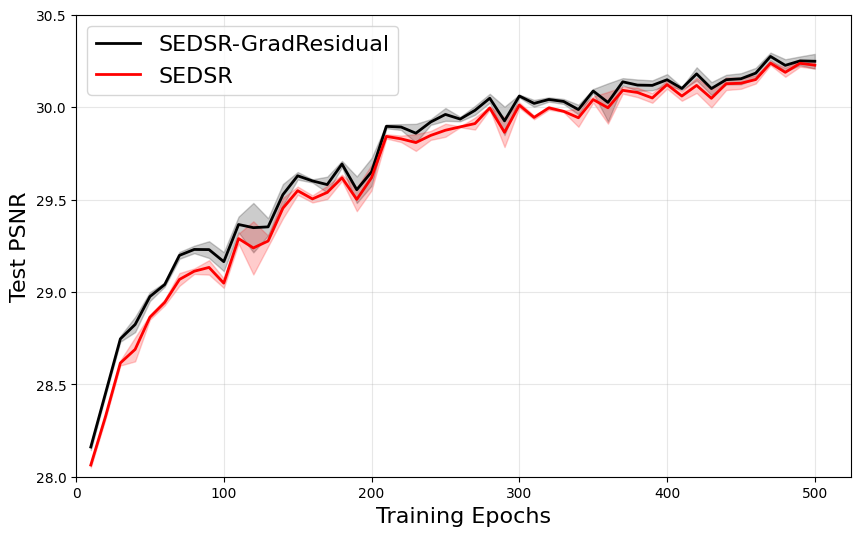}
\caption{Urban100}
\end{subfigure}
\caption{Learning curves of test performance measured by PSNR v.s. training epochs on standard benchmarks. The first row is with batch size = 32, and the second row is with batch size = 16. We train for 500 epochs in total and evaluate every 10 epochs. The results are averaged over 3 random seeds. 
}
\label{fig:sedsr-batchsize}
\end{figure*}

\textbf{The effect of batch normalization.} Table \ref{tab:srresnet_best} and \ref{tab:srresnet_final} show the PSNR result with and without batch normalization when using SRResNet as base architecture. It can be clearly observed that for both residual types, removing batch normalization could be beneficial. 

\begin{table}[ht]
\centering
\begin{tabular}{lcccc}
\toprule
Dataset & SRResNet-NoBN & SRResNet-GradResidual-NoBN & SRResNet & SRResNet-GradResidual \\
\midrule
Urban100 & $29.41 \pm 0.006121$ & $29.39 \pm 0.03227$ & $28.79 \pm 0.02725$ & $28.67 \pm 0.1009$ \\
BSD100   & $31.45 \pm 0.01461$ & $31.43 \pm 0.01831$ & $31.19 \pm 0.01260$ & $31.12 \pm 0.04933$ \\
Set5     & $36.78 \pm 0.02131$ & $36.76 \pm 0.03938$ & $36.16 \pm 0.01422$ & $36.01 \pm 0.1358$ \\
Set14    & $32.59 \pm 0.03810$ & $32.58 \pm 0.02599$ & $32.19 \pm 0.02241$ & $32.12 \pm 0.06223$ \\
DIV2K    & $34.85 \pm 0.009897$ & $34.82 \pm 0.04140$ & $34.40 \pm 0.03725$ & $34.29 \pm 0.08452$ \\
\bottomrule
\end{tabular}
\caption{Best PSNR $\pm$ std.\ error.}\label{tab:srresnet_best}
\end{table}

\begin{table*}[ht]
\centering
\begin{tabular}{lcccc}
\toprule
Dataset & SRResNet-NoBN & SRResNet-GradResidual-NoBN & SRResNet & SRResNet-GradResidual \\
\midrule
Urban100 & $29.26 \pm 0.01497$ & $29.24 \pm 0.02788$ & $28.67 \pm 0.01212$ & $28.52 \pm 0.1280$ \\
BSD100   & $31.40 \pm 0.02033$ & $31.34 \pm 0.02539$ & $31.11 \pm 0.007074$ & $31.04 \pm 0.06080$ \\
Set5     & $36.63 \pm 0.02720$ & $36.52 \pm 0.04348$ & $35.96 \pm 0.08048$ & $35.79 \pm 0.1715$ \\
Set14    & $32.52 \pm 0.01642$ & $32.46 \pm 0.03249$ & $32.10 \pm 0.01223$ & $32.02 \pm 0.07265$ \\
DIV2K    & $34.72 \pm 0.04557$ & $34.60 \pm 0.05540$ & $34.21 \pm 0.006017$ & $34.08 \pm 0.06475$ \\
\bottomrule
\end{tabular}
\caption{Average final PSNR $\pm$ std.\ error. }\label{tab:srresnet_final}
\end{table*}

\textbf{Examples of generated images.} We show some examples of generated images from standard testing benchmarks in Figures~\ref{fig:sr_gen_urban100}, \ref{fig:sr_gen_bsd100}, \ref{fig:sr_gen_set5}, \ref{fig:sr_gen_set14}, and \ref{fig:sr_gen_div2k}. We do not include all generated images because of their large storage requirements (on the order of gigabytes) and because they are less informative than the reported PSNR results and learning curves.

\begin{figure}[t]
    \centering
    \includegraphics[width=0.8\textwidth]{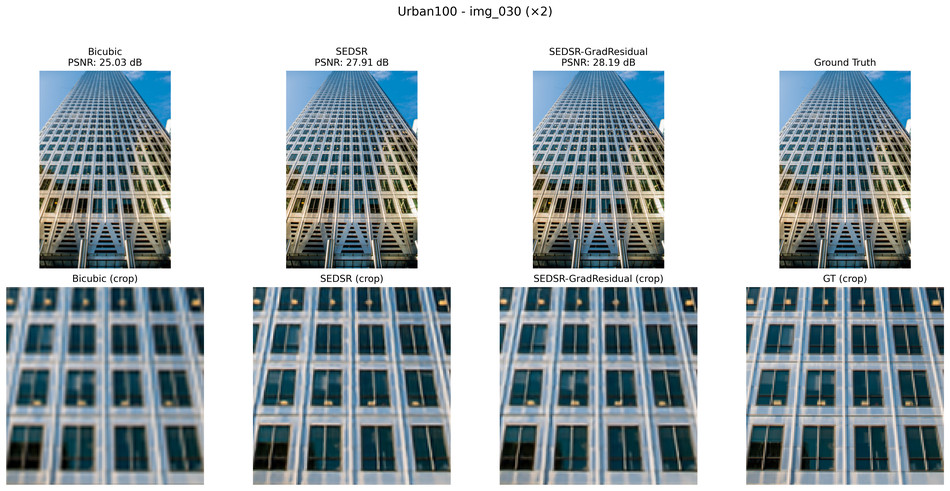}
    \caption{Urban100. From left to right, the images are generated using bicubic interpolation, SEDSR, SEDSR-GradResidual, and the ground truth.}
    \label{fig:sr_gen_urban100}
\end{figure}

\begin{figure}[t]
    \centering
    \includegraphics[width=0.8\textwidth]{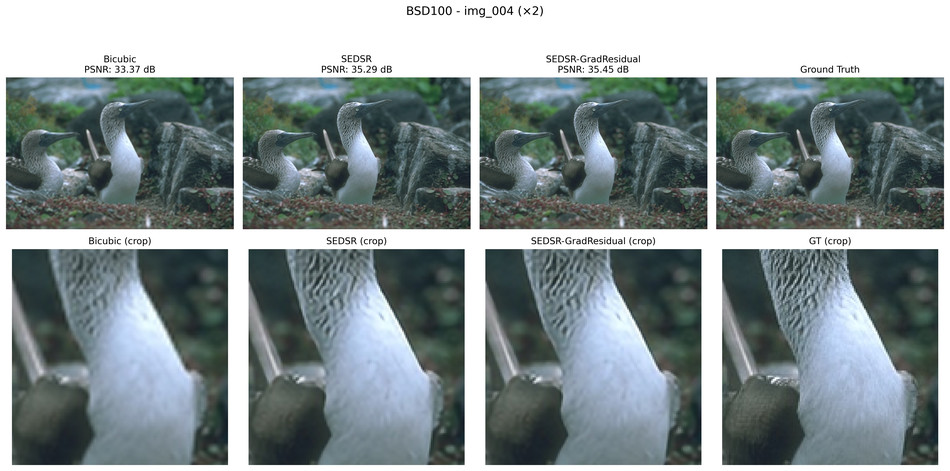}
    \caption{BSD100. From left to right, the images are generated using bicubic interpolation, SEDSR, SEDSR-GradResidual, and the ground truth.}
    \label{fig:sr_gen_bsd100}
\end{figure}

\begin{figure}[t]
    \centering
    \includegraphics[width=0.8\textwidth]{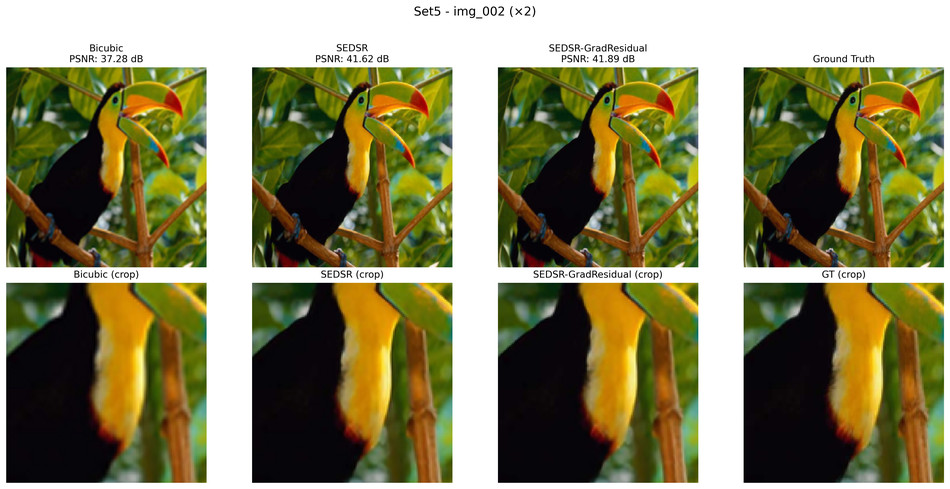}
    \caption{Set5. From left to right, the images are generated using bicubic interpolation, SEDSR, SEDSR-GradResidual, and the ground truth.}
    \label{fig:sr_gen_set5}
\end{figure}

\begin{figure}[t]
    \centering
    \includegraphics[width=0.8\textwidth]{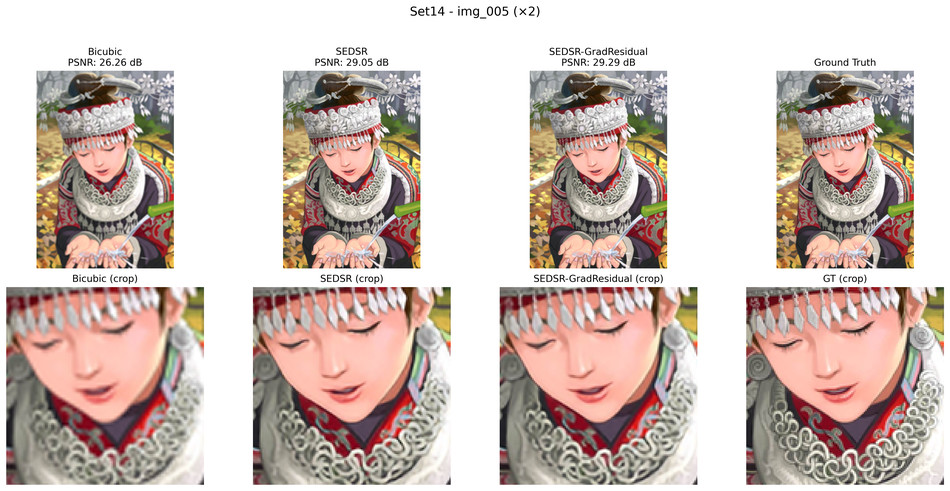}
    \caption{Set14. From left to right, the images are generated using bicubic interpolation, SEDSR, SEDSR-GradResidual, and the ground truth.}
    \label{fig:sr_gen_set14}
\end{figure}

\begin{figure}[t]
    \centering
    \includegraphics[width=0.8\textwidth]{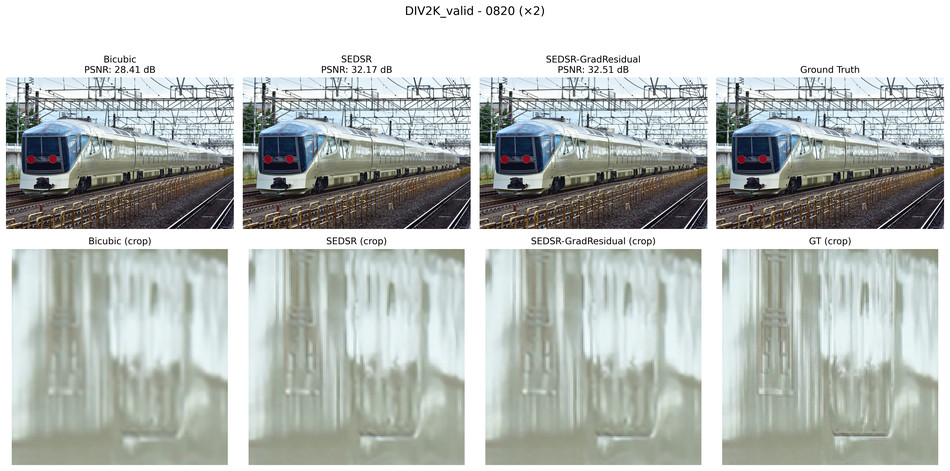}
    \caption{Div2k validation. From left to right, the images are generated using bicubic interpolation, SEDSR, SEDSR-GradResidual, and the ground truth.}
    \label{fig:sr_gen_div2k}
\end{figure}

\subsection{Additional Gradient Residual Design on Super Resolution}\label{sec:appendix-variants-sr}

Recall that our proposed version is $\hvec(\xvec) \defeq \Fvec(\xvec) + (1-\sigma(\alpha)) \xvec + \sigma(\alpha) \sum_{i=1}^d \nabla  \Fvec_i(\xvec)$. We test the below two additional residual designs on super resolution tasks: 

1) In the first design, we define two independent trainable scalars instead of using the convex combination:
\begin{equation}\label{eq:independent-gradscale}
   \hvec(\xvec) \defeq \Fvec(\xvec) + \sigma(\beta) \xvec + \sigma(\alpha) \sum_{i=1}^d \nabla  \Fvec_i(\xvec) 
\end{equation}
and hence the baseline is using: 
\begin{equation}\label{eq:independent-resscale}
   \hvec(\xvec) \defeq \Fvec(\xvec) + \sigma(\beta) \xvec. 
\end{equation}

2) In the second design, we parameterize the two trainable scalars so that they become sample-dependent:
\begin{equation}\label{eq:sample-dependent-gradscale}
   \hvec(\xvec) \defeq \Fvec(\xvec) + \sigma(\beta(\xvec)) \xvec + \sigma(\alpha(\xvec)) \sum_{i=1}^d \nabla  \Fvec_i(\xvec), 
\end{equation}
where $\alpha(\cdot)$ and $\beta(\cdot)$ are parameterized in the same way but with distinct sets of parameters. And the baseline is using:
\begin{equation}\label{eq:sample-dependent-resscale}
   \hvec(\xvec) \defeq \Fvec(\xvec) + \sigma(\beta(\xvec)) \xvec. 
\end{equation}
Such additional layer consists of a global pooling operation followed by a linear mapping and then a sigmoid transformation. The rationale behind this design is that not all samples lie in high-frequency regions; therefore, only a subset of samples may require a larger scaling factor on the gradient component.


For variant (2), the overall observations differ from those by the convex combination setting in Section \ref{sec:imageexp}. When SRResNet is used as the base architecture (Table \ref{tab:srresnet_samplebasedscale_best} and \ref{tab:srresnet_samplebasedscalar_final}), the gradient residual significantly outperforms the baseline SRResNet model, which contrasts with our earlier observations under the convex combination formulation. When SEDSR is used as the base architecture, the gradient residual variant shows a mild improvement (Table \ref{tab:sedsr_samplebasedscale_best} and \ref{tab:sedsr_samplebasedscale_final}), whereas for EDSR, the baseline model with standard residual connections exhibits a slight advantage over the gradient residual variant (Table \ref{tab:edsr_samplebasedscale_best} and \ref{tab:edsr_samplebasedscale_final}).

These inconsistent results may indicate that many existing deep architectures are highly optimized around standard residual connections. This suggests that further investigation is warranted to understand how gradient residual connections can be effectively integrated into existing architectures, or how network designs and training protocols should be adapted when introducing gradient-based connections.

\begin{table}[ht]
\centering
\begin{tabular}{lcc}
\toprule
Dataset & SRResNet & SRResNet-GradResidual \\
\midrule
Urban100 & $27.72 \pm 0.009831$ & $\mathbf{28.52 \pm 0.03805}$ \\
BSD100   & $30.61 \pm 0.01745$  & $\mathbf{31.08 \pm 0.01863}$ \\
Set5     & $35.11 \pm 0.06085$  & $\mathbf{35.90 \pm 0.03865}$ \\
Set14    & $31.55 \pm 0.02016$  & $\mathbf{31.96 \pm 0.02154}$ \\
DIV2K    & $33.59 \pm 0.04453$  & $\mathbf{34.27 \pm 0.01925}$ \\
\bottomrule
\end{tabular}
\caption{Best evaluation PSNR $\pm$ std.\ error for SRResNet with sample-based scaling. The best is defined as the best evaluation result averaged over 3 random seeds during training.}
\label{tab:srresnet_samplebasedscale_best}
\end{table}

\begin{table}[ht]
\centering
\begin{tabular}{lcc}
\toprule
Dataset & SEDSR & SEDSR-GradResidual \\
\midrule
Urban100 & $30.13 \pm 0.008821$ & $\mathbf{30.18 \pm 0.007823}$ \\
BSD100   & $31.75 \pm 0.004790$ & $\mathbf{31.76 \pm 0.003564}$ \\
Set5     & $37.27 \pm 0.007073$ & $\mathbf{37.30 \pm 0.002883}$ \\
Set14    & $32.99 \pm 0.004689$ & $\mathbf{33.02 \pm 0.002467}$ \\
DIV2K    & $35.36 \pm 0.004687$ & $\mathbf{35.37 \pm 0.005749}$ \\
\bottomrule
\end{tabular}
\caption{Average final PSNR $\pm$ std.\ error for SEDSR with sample-based scaling. Residual connections are based on \eqref{eq:sample-dependent-gradscale} and \eqref{eq:sample-dependent-resscale}.}
\label{tab:sedsr_samplebasedscale_final}
\end{table}

\begin{table}[ht]
\centering
\begin{tabular}{lcc}
\toprule
Dataset & SEDSR & SEDSR-GradResidual \\
\midrule
Urban100 & $30.18 \pm 0.02804$ & $\mathbf{30.25 \pm 0.03192}$ \\
BSD100   & $31.77 \pm 0.01025$ & $\mathbf{31.79 \pm 0.003427}$ \\
Set5     & $37.33 \pm 0.01449$ & $\mathbf{37.35 \pm 0.003476}$ \\
Set14    & $33.02 \pm 0.007811$ & $\mathbf{33.06 \pm 0.007423}$ \\
DIV2K    & $35.39 \pm 0.01478$ & $\mathbf{35.42 \pm 0.006848}$ \\
\bottomrule
\end{tabular}
\caption{Best evaluation PSNR $\pm$ std.\ error for SEDSR with sample-based scaling. Residual connections are based on \eqref{eq:sample-dependent-gradscale} and \eqref{eq:sample-dependent-resscale}.}
\label{tab:sedsr_samplebasedscale_best}
\end{table}

\begin{table}[ht]
\centering
\begin{tabular}{lcc}
\toprule
Dataset & EDSR & EDSR-GradResidual \\
\midrule
Urban100 & $\mathbf{29.99 \pm 0.04601}$ & $29.88 \pm 0.02466$ \\
BSD100   & $\mathbf{31.74 \pm 0.01964}$ & $31.71 \pm 0.01474$ \\
Set5     & $\mathbf{37.24 \pm 0.01501}$ & $37.22 \pm 0.03968$ \\
Set14    & $\mathbf{32.95 \pm 0.02239}$ & $32.91 \pm 0.01818$ \\
DIV2K    & $\mathbf{35.26 \pm 0.02206}$ & $35.22 \pm 0.02637$ \\
\bottomrule
\end{tabular}
\caption{Average final PSNR $\pm$ std.\ error for EDSR with sample-based scaling. Residual connections are based on \eqref{eq:sample-dependent-gradscale} and \eqref{eq:sample-dependent-resscale}.}
\label{tab:edsr_samplebasedscale_final}
\end{table}

\begin{table}[ht]
\centering
\begin{tabular}{lcc}
\toprule
Dataset & EDSR & EDSR-GradResidual \\
\midrule
Urban100 & $\mathbf{30.10 \pm 0.002069}$ & $30.00 \pm 0.02343$ \\
BSD100   & $\mathbf{31.78 \pm 0.01213}$  & $31.75 \pm 0.009103$ \\
Set5     & $\mathbf{37.35 \pm 0.02212}$  & $37.34 \pm 0.03808$ \\
Set14    & $\mathbf{33.01 \pm 0.01659}$  & $32.97 \pm 0.01383$ \\
DIV2K    & $\mathbf{35.34 \pm 0.01257}$  & $35.29 \pm 0.01514$ \\
\bottomrule
\end{tabular}
\caption{Best evaluation PSNR $\pm$ std.\ error for EDSR with sample-based scaling. Residual connections are based on \eqref{eq:sample-dependent-gradscale} and \eqref{eq:sample-dependent-resscale}.}
\label{tab:edsr_samplebasedscale_best}
\end{table}

\begin{table}[ht]
\centering
\begin{tabular}{lcc}
\toprule
Dataset & SEDSR & SEDSR-GradResidual \\
\midrule
urban100 & $30.07 \pm 0.02539$ & $30.08 \pm 0.008159$ \\
bsd100   & $31.73 \pm 0.009620$ & $31.73 \pm 0.003866$ \\
set5     & $37.25 \pm 0.02782$  & $37.24 \pm 0.01426$ \\
set14    & $32.96 \pm 0.01746$  & $32.96 \pm 0.006714$ \\
div2k    & $35.31 \pm 0.01859$  & $35.31 \pm 0.007634$ \\
\bottomrule
\end{tabular}
\caption{Final evaluation based on SEDSR. Trainable residual scalar.}
\label{tab:sedsr-final-trainres}
\end{table}

\begin{table}[ht]
\centering
\begin{tabular}{lcc}
\toprule
Dataset & SEDSR & SEDSR-GradResidual \\
\midrule
urban100 & $30.14 \pm 0.03268$ & $30.15 \pm 0.02692$ \\
bsd100   & $31.75 \pm 0.007000$ & $31.76 \pm 0.004182$ \\
set5     & $37.33 \pm 0.03642$  & $37.31 \pm 0.02172$ \\
set14    & $32.99 \pm 0.02072$  & $33.01 \pm 0.01227$ \\
div2k    & $35.34 \pm 0.01703$  & $35.35 \pm 0.01379$ \\
\bottomrule
\end{tabular}
\caption{Best evaluation based on SEDSR. Trainable residual scalar.}
\label{tab:sedsr-best-trainres}
\end{table}

\begin{table}[ht]
\centering
\begin{tabular}{lcc}
\toprule
Dataset & EDSR & EDSR-GradResidual \\
\midrule
urban100 & $30.00 \pm 0.01045$   & $30.00 \pm 0.008629$ \\
bsd100   & $31.71 \pm 0.0003460$ & $31.71 \pm 0.003574$ \\
set5     & $37.21 \pm 0.02321$   & $37.22 \pm 0.01643$ \\
set14    & $32.92 \pm 0.008248$  & $32.92 \pm 0.005228$ \\
div2k    & $35.27 \pm 0.01657$   & $35.27 \pm 0.007332$ \\
\bottomrule
\end{tabular}
\caption{Final evaluation based on EDSR, Trainable residual scalar.}
\label{tab:edsr-final-trainres}
\end{table}

\begin{table}[ht]
\centering
\begin{tabular}{lcc}
\toprule
Dataset & EDSR & EDSR-GradResidual \\
\midrule
urban100 & $30.13 \pm 0.007543$ & $30.13 \pm 0.001675$ \\
bsd100   & $31.76 \pm 0.007424$ & $31.76 \pm 0.008123$ \\
set5     & $37.32 \pm 0.01352$  & $37.32 \pm 0.007590$ \\
set14    & $32.99 \pm 0.01038$  & $32.98 \pm 0.009934$ \\
div2k    & $35.35 \pm 0.03103$  & $35.37 \pm 0.008518$ \\
\bottomrule
\end{tabular}
\caption{Peak evaluation for EDSR. Trainable residual scalar. }
\label{tab:edsr-best-trainres}
\end{table}

\begin{table}[ht]
\centering
\begin{tabular}{lcc}
\toprule
Dataset & SRResNet & SRResNet-GradResidual \\
\midrule
urban100 & $28.64 \pm 0.02679$ & $28.63 \pm 0.03761$ \\
bsd100   & $31.10 \pm 0.01198$ & $31.10 \pm 0.01542$ \\
set5     & $35.93 \pm 0.05233$ & $35.90 \pm 0.08888$ \\
set14    & $32.08 \pm 0.02400$ & $32.07 \pm 0.02621$ \\
div2k    & $34.13 \pm 0.06527$ & $34.14 \pm 0.03375$ \\
\bottomrule
\end{tabular}
\caption{Final evaluation for SRResNet. Trainable residual scalar.}
\label{tab:srresnet-final-trainres}
\end{table}

\begin{table}[ht]
\centering
\begin{tabular}{lcc}
\toprule
Dataset & SRResNet & SRResNet-GradResidual \\
\midrule
urban100 & $28.78 \pm 0.01764$ & $28.76 \pm 0.03504$ \\
bsd100   & $31.20 \pm 0.01389$ & $31.17 \pm 0.01894$ \\
set5     & $36.17 \pm 0.03770$ & $36.12 \pm 0.08022$ \\
set14    & $32.20 \pm 0.02628$ & $32.17 \pm 0.03300$ \\
div2k    & $34.41 \pm 0.05192$ & $34.35 \pm 0.05427$ \\
\bottomrule
\end{tabular}
\caption{Peak evaluation for SRResNet. Trainable residual scalar.}
\label{tab:srresnet-best-trainres}
\end{table}

\subsection{Classification and Segmentation}\label{sec:appendix-classify-seg}

On CIFAR-10/100 \cite{krizhevsky2009learning}, we use ResNet20 \cite{he2016deepresidual} as base model. We sweep the learning rate over ${0.05, 0.1, 0.2}$ and the weight decay over $\{0.0, 5\times10^{-5}, 1\times10^{-4}\}$. Models are trained for $200$ epochs using SGD and evaluated every $5$ epochs. For our gradient-based variant, we additionally sweep the convex combination scalar parameter $\alpha$ over the range $\{-3, 3\}$. Experiments on this domain are conducted using $5$ random seeds. 

For the semantic segmentation task, we use the PASCAL VOC dataset \cite{everingham2010pascal} and adopt a simple encoder–decoder architecture with a pretrained MobileNetV3-Large backbone extracted from the torchvision LRASPP implementation \cite{howard2019mobilenetv3}. We use a batch size of $32$ and train for $200$ epochs, with evaluation performed every $5$ epochs. All models are optimized using stochastic gradient descent with momentum set to $0.9$, and we sweep the weight decay over $\{1\times 10^{-4}, 5\times10^{-4}\}$. Experiments on this domain are conducted using $3$ random seeds. The mean Intersection-over-Union (mIoU) results are reported in Table~\ref{tab:class_seg_results}. For class $c$, let $P_c$ be the set of pixels predicted as class $c$ and $G_c$ be the set of ground-truth pixels labeled as class $c$ (excluding ignored pixels). The Intersection-over-Union for class $c$ is
$$
\mathrm{IoU}_c = \frac{|P_c \cap G_c|}{|P_c \cup G_c|}.
$$
Let $\mathcal{C}$ denote the set of classes used for evaluation. The mIoU is
$$
\mathrm{mIoU} = \frac{1}{|\mathcal{C}|}\sum_{c \in \mathcal{C}} \mathrm{IoU}_c.
$$

\begin{table}[t]
\centering
\begin{tabular}{lcc}
\toprule
\textbf{Dataset / Metric} 
& \textbf{W.O. GradResidual} 
& \textbf{W. GradResidual} \\
\midrule
CIFAR-10 (Final)
& $0.9183 \pm 0.0010$
& $0.9158 \pm 0.0014$ \\

CIFAR-10 (Best)
& $0.9187 \pm 0.0007$
& $0.9169 \pm 0.0013$ \\

CIFAR-100 (Final)
& $0.6688 \pm 0.0018$
& $0.6661 \pm 0.0004$ \\

CIFAR-100 (Best)
& $0.6753 \pm 0.0011$
& $0.6734 \pm 0.0010$ \\

VOC (Final)
& $0.6211 \pm 0.0006$
& $0.6237 \pm 0.0010$ \\

VOC (Best)
& $0.6299 \pm 0.0030$
& $0.6325 \pm 0.0009$ \\
\bottomrule
\end{tabular}
\caption{
Comparison of models with and without GradResidual across datasets.
On CIFAR-10 and CIFAR-100, accuracy is reported, while on VOC, mean intersection-over union (mIoU) is reported.
Final denotes the average over final $25\%$ evaluations, and Best denotes the peak performance during training.
CIFAR results are averaged over 5 random seeds and VOC results over 3 random seeds.
For GradResidual variants, the gradient scale initialization is $\alpha = -3.0$.
}
\label{tab:class_seg_results}
\vspace{-0.3cm}
\end{table}

\subsection{Additional Literature Review}\label{sec:appendix-relatedwork}

To the best of our knowledge, the observation from our synthetic experiments that residual connections may hinder the approximation of high-frequency functions has received limited direct attention in the literature. Although some prior work has explored potential connections between frequency characteristics and residual connections, these studies are only distantly related to our findings. For example, \citet{belfer2024resntk} studies the spectral properties of the neural tangent kernel associated with residual networks, such as eigenvalues, eigenfunctions, and smoothness, providing insights into their learning and generalization behavior.

We also note that since the original introduction of residual networks, numerous variants, adaptations, and extensions have been proposed to address a wide range of objectives. However, these works are typically motivated by orthogonal considerations or adopt fundamentally different perspectives and approaches from ours. For example, \citet{zagoruyko2016wideresnet} argues that making networks wider is often more worthwhile, giving better accuracy and faster training than simply going deeper. \citet{xie2017resnext} introduces cardinality (parallel grouped convolutions inside a residual block), which achieves strong empirical performance. Inception-ResNet \cite{szegedy2017inceptionv4} integrates Inception modules with residual connections, and shows that residual connections can significantly speed up optimization for very deep and sophisticated CNNs. SENet and CBAM \cite{hu2018senet,woo2018cbam} introduce lightweight attention into residual backbones.

Several works focus on improving robustness/regularization. \citet{gastaldi2017shakeshake} randomly mixes multi-branch outputs during training, injecting structured noise that improves generalization. Similarly, \citet{huang2016stochasticdepth,yamada2019shakedrop} randomly drops or perturbs residual branches to improve robustness and training efficiency. From a more theoretical viewpoint, Neural ODEs \cite{chen2018neuralode} view residual connections as continuous-time dynamics and use an ODE solver instead of discrete layers. For low-level vision, \citet{lim2017edsr,wang2018esrgan} propose architecture adjustments for super-resolution, such as removing batch normalization and modifying/scaling residual blocks. Finally, \citet{sandler2018mobilenetv2} adapts the residual idea into an efficient design (inverted residuals and linear bottlenecks) that is particularly suitable for mobile/edge devices while preserving performance.

\end{document}